%% file: uai2022.tex
\documentclass[accepted]{uai2022} 

\usepackage[american]{babel}

\usepackage{natbib} 
\usepackage{mathtools} 
\usepackage{booktabs} 
\usepackage{tikz} 

\usepackage{xcolor}
\usepackage{amsmath,amssymb,amsthm,mathrsfs,url,array}
\usepackage{graphicx}
\usepackage{subfigure}
\usepackage{wrapfig}
\usepackage{appendix}
\usepackage{multirow}
\usepackage{makecell}
\usepackage{diagbox}
\usepackage[ruled,linesnumbered]{algorithm2e}

\newtheorem{Th}{Theorem}

\newtheorem{Lm}{Lemma}
 
\allowdisplaybreaks[4]

\title{X-MEN: Guaranteed XOR-Maximum Entropy \\
Constrained Inverse Reinforcement Learning}

\author[1]{\href{mailto:<ding274@purdue.edu>?Subject=Your UAI 2022 paper}{Fan Ding}{}}
\author[1]{Yexiang Xue}
\affil[1]{%
    Computer Science Dept.\\
    Purdue University\\
    West Lafayette, Indiana, USA
}
  
\begin{document}
\maketitle

\input{tex/abstract}

\input{tex/intro}

\input{tex/prelim}

\input{tex/method}

\input{tex/related}

\input{tex/exp}

\input{tex/conclusion}




\clearpage
\bibliography{fan}


\clearpage
\appendix
\input{tex/appendix}

\end{document}

%% file: tex/abstract.tex
\begin{abstract}
Inverse Reinforcement Learning (IRL) is a powerful way of learning from demonstrations. 
In this paper, we address IRL problems with the availability of prior knowledge that optimal policies will never violate certain constraints. 
Conventional approaches ignoring these constraints need many demonstrations to converge. 
We propose XOR-Maximum Entropy Constrained Inverse Reinforcement Learning (X-MEN), which is guaranteed to converge to the optimal policy in linear rate w.r.t. the number of learning iterations. 
X-MEN embeds XOR-sampling -- a provable sampling approach which transforms the \#-P complete sampling problem into queries to NP oracles -- into the framework of maximum entropy IRL. 
%
%
X-MEN also guarantees the learned policy will never generate trajectories that violate constraints. 
Empirical results in navigation demonstrate that X-MEN converges faster to the optimal policies compared to baseline approaches and always generates trajectories that satisfy multi-state combinatorial  constraints. 
%
%
%
%
\end{abstract}

%% file: tex/intro.tex
\section{INTRODUCTION}
\begin{figure*}[t]
\subfigure[Add no constraint]{\label{fig:intuition1}
\includegraphics[width=0.32\linewidth]{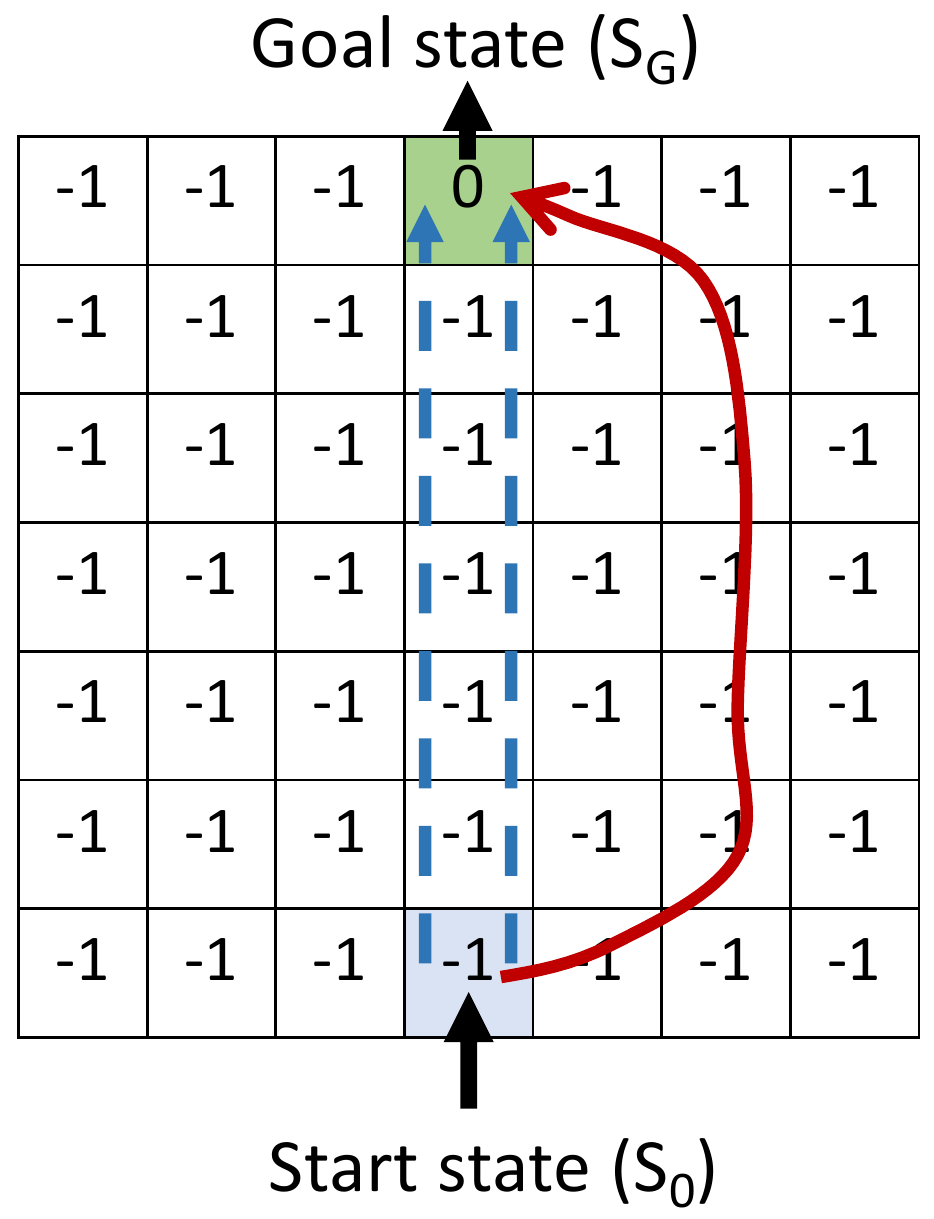}}
\subfigure[Add single-state constraints $\mathcal{C}_1$]{\label{fig:intuition2}
\includegraphics[width=0.32\linewidth]{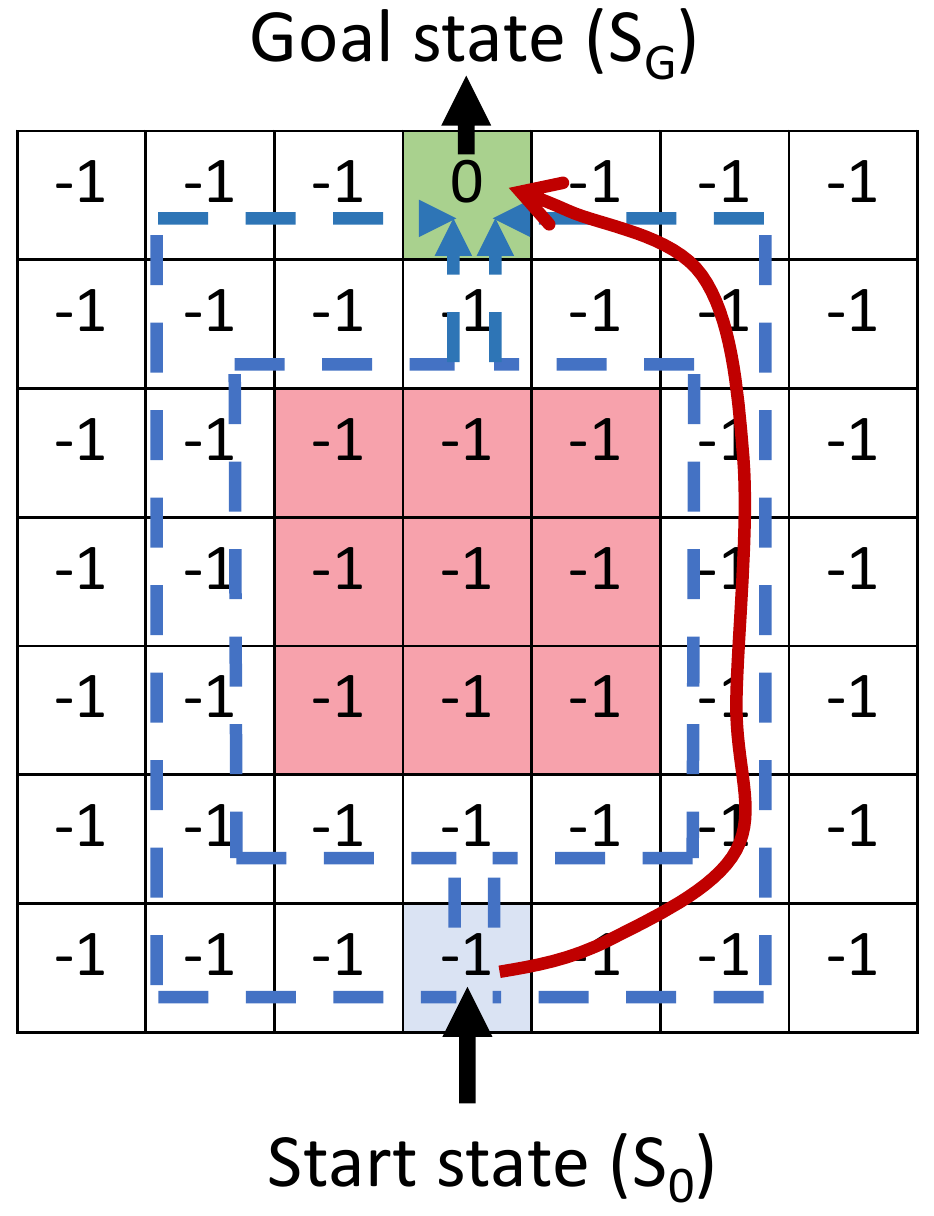}}
\subfigure[Add multi-state  constraint $\mathcal{C}_2$]{\label{fig:intuition3}
\includegraphics[width=0.32\linewidth]{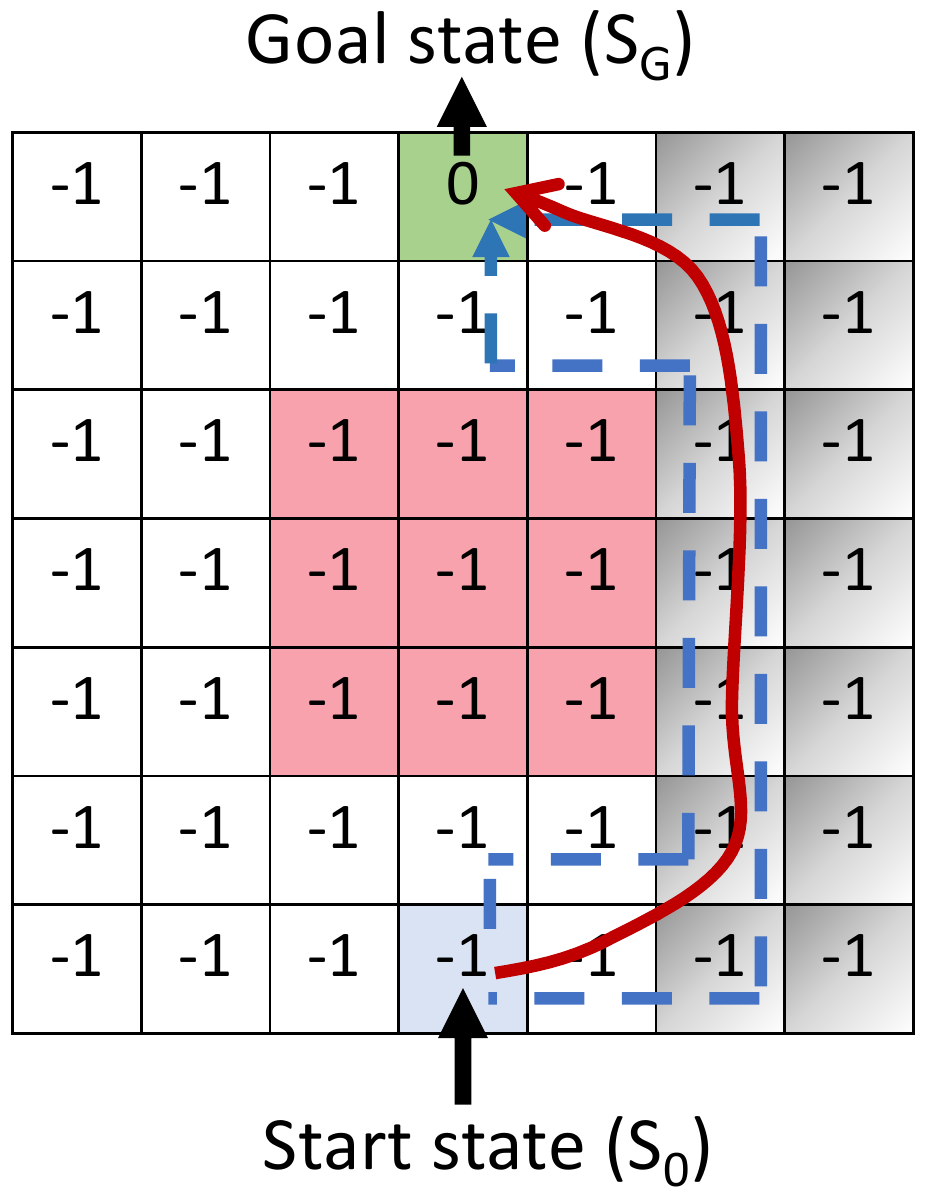}}
\caption{Examples of constrained IRL problems. The agent wants to move from the start state $S_0$ (blue grid) to the goal state $S_G$ (green grid). Ground truth demonstration is shown in the red line. The same initial reward function before learning is used for all 3 situations, with one-step reward listed in each grid. Most likely trajectories under the initial reward function (e.g.,those maximizing rewards and subject to constraints) are shown using blue dashed arrows. (\textbf{a}) When no constraint is added to the MDP, the agent finds the shortest path directly upward from $S_0$ to $S_G$. (\textbf{b}) When single-state constraint $\mathcal{C}_1$, which forbids the agent to  to pass through the red grids, is imposed, the agent can detour from either the  left or the right side. (\textbf{c}) When there are an additional multi-state constraint $\mathcal{C}_2$ imposed, which constrains at least half of all  passing states in the shaded area,   the optimal trajectory is to detour from the right side. Notice that this behavior aligns with the demonstration.} 
\label{fig:intuition}
\end{figure*}

Inverse Reinforcement Learning (IRL) \citep{ng2000algorithms,abbeel2004apprenticeship,ziebart2008maximum,arora2021survey,li2017deep}
provides an important way to learn from demonstrations. 
IRL assumes that the demonstrator implicitly maximizes the cumulative reward of a Markov Decision Process (MDP). 
The goal of IRL is to recover the unknown reward function from the observed demonstrations. 
%
Various IRL algorithms have been proposed, including Linear IRL \citep{ng2000algorithms,abbeel2004apprenticeship} and Large-Margin Q-Learning \citep{ratliff2006maximum}. 
To differentiate among multiple reward functions which lead to similar behaviors, Maximum Entropy Inverse Reinforcement Learning (MaxEnt IRL) \citep{ziebart2008maximum,wulfmeier2015maximum,finn2016guided,ho2016generative} assumes that the demonstrator samples trajectories from a maximum entropy distribution parameterized by the cumulative reward. 

In this paper, we focus on IRL problems where certain constraints are known beforehand and hence do not need to be rediscovered by the learning algorithm. 
The trajectories from the demonstrator are known to satisfy these constraints and we require the IRL agent to follow trajectories satisfying these constraints as well. 
Indeed, standard IRL algorithms \citep{abbeel2007application,vasquez2014inverse,scobee2018haptic} can be applied to this scenario without modifications and they eventually discover the optimal policy. 
Nevertheless, it may require a large amount of demonstrations to learn these constraints. 
Worse still, it is still possible for the IRL agent to produce trajectories which occasionally violate constraints even after many training epochs.
This is especially problematic in safety critical domains, such as autonomous driving, robotic surgery, etc. 


Recent work has attempted to embed constraints into IRL. For example, the work of \citep{vazquez2017learning,kalweit2020deep}  uses demonstrations to learn a rich class of possible specifications that can represent a task. Others have focused specifically on learning constraints, that is, behaviors that are expressly forbidden or infeasible \citep{chou2018learning,subramani2018inferring,mcpherson2018modeling,scobee2019maximum,anwar2020inverse,mcpherson2021maximum}. 
Nevertheless, so far the attempts have been focused on \textit{single-state} constraints, where a handful of actions are forbidden in certain states and these forbidden actions have little impact for future state-action transitions. Their approaches  cannot address \textit{multi-state combinatorial} constraints, which limits a chain of actions spanning multiple time stamps. 
For example, Figure~\ref{fig:intuition} (c) demonstrates a navigation task where constraints require at least half of the states in each trajectory is located in the shaded area. 
With this constraint imposed, only trajectories passing the right-hand side are possible. 
Such constraints cannot be addressed with previous approaches which mask out actions from certain states. 
 

In this work, we propose \textbf{X}OR-\textbf{M}aximum \textbf{EN}tropy (X-MEN) Constrained Inverse Reinforcement Learning, which \textbf{\textit{provably converges to the optimal reward function for MaxEnt IRL in linear number of training steps}}, even in the presence of hard combinatorial constraints. 
X-MEN also guarantees to produce trajectories which satisfy multi-state combinatorial constraints. 
X-MEN is based on the Maximum Entropy IRL learning \citep{ziebart2008maximum,boularias2011relative}. 
Distinctively, X-MEN harnesses XOR-sampling to estimate the  gradient of the expected reward from the current model distribution. 
The recently proposed XOR-Sampling \citep{Gomes2006NearUniformSampling,Ermon13Wish,ermon2013embed} reduces the sampling problem into queries of NP oracles via hashing and projection, and guarantees a constant factor approximation for the expectation estimation. 
After obtaining samples, X-MEN uses Stochastic Gradient Descent (SGD) to maximize the difference between expected reward from the demonstration and from the trajectories sampled from the current model distribution, a procedure closely resembling contrastive divergence learning, to maximize the likelihood of the demonstrated behavior. 
Theoretic analysis reveals that X-MEN provably converges to the \textit{global optimum} of the likelihood function in linear number of SGD iterations. 
%
%
In addition, X-MEN can handle rewards parameterized either in a linear form or in the representation of a neural network.
During testing, the policy learned by X-MEN can also be adapted to satisfying additional constraints without retraining. 


In experiments, we compare the performance of X-MEN against MaxEnt IRL \citep{ziebart2008maximum} and additional baselines such as Reletive Entropy IRL (RE-IRL) \citep{boularias2011relative} and recently proposed maximum likelihood constraint inference (MLCI) \citep{scobee2019maximum} on several grid world environments and in an imitation learning environment with human data of navigating around obstacles. All these environments require the agent to follow constraints. 
Our experiment shows the learned trajectories of X-MEN 100\% satisfy constraints, while a majority of trajectories produced by competing approaches do not ($\geq 60\%$ violate constraints). 
Also X-MEN produces trajectories that closely imitate those of the demonstrations. 
In summary, our contributions are as follows:
\begin{itemize}
    \item We propose X-MEN, an algorithm that provably converges to the optimal reward function for MaxEnt IRL in linear number of training steps, even in the presence of multi-state combinatorial constraints.
    \item X-MEN is guaranteed to produce trajectories which satisfy combinatorial constraints, beyond the capability of previous approaches. 
    \item Experimental results reveal that X-MEN produces trajectories that closely resemble demonstration while satisfying constraints, outperforming a series of constrained IRL baselines. 
\end{itemize}

%% file: tex/prelim.tex

\section{INVERSE REINFORCEMENT LEARNING}
Here we present a brief overview of IRL.   $\mathcal{M}=\{\mathcal{S}, \mathcal{A}, T, R, \gamma\}$ is a Markov Decision Process (MDP), where $\mathcal{S}$ denotes the state space of all states $s$, $\mathcal{A}$ denotes the set of possible actions $a$, $T$ denotes the transition probability function, $R$ denotes the reward function, and $\gamma \in [0, 1]$ is the  discount factor. Given an MDP, an optimal policy $\pi^*$ is the one to maximize the expected cumulative reward. 
IRL considers the case where  the reward function is unknown. Instead, a set of expert demonstrations $\mathcal{D}=\{\tau_1,\ldots, \tau_N\}$ is provided which are sampled from a user policy $\pi$, i.e. provided by a demonstrator. Each demonstration
consists of a series of state-action pairs $\tau_i=\{(s_{i0},a_{i0}), (s_{i1},a_{i1}),\ldots,(s_{iL},a_{iL_{i}})\}$, where $L_i$ denotes the length of the trajectory. The goal of IRL is to uncover the hidden reward $r$ from the demonstrations.

\subsection{Maximum Entropy IRL}
A number of approaches have been proposed to tackle the IRL problem \citep{ng2000algorithms,abbeel2004apprenticeship,ratliff2006maximum}. One crucial problem to address for IRL is to differentiate among multiple reward functions that lead to the same demonstrations. 
An influential formulation is Maximum Entropy IRL \citep{ziebart2008maximum}, which can also be viewed as a special case of Relative Entropy IRL (RE-IRL) \citep{boularias2011relative,snoswell2020revisiting}. 
In this formulation, the probability that the demonstrator chooses a given trajectory is proportional to the exponent of the reward along the path. Denote  $R_{\theta}(\tau)=\sum_{t=1}^{L}\gamma^t R_{\theta}(s_t,a_t)$ as the discounted cumulative reward  parameterized by $\theta$. The probability of choosing trajectory $\tau$ is:
\begin{align}
    P_{choice}(\tau|\theta,T) = \frac{1}{Z_{\theta}}e^{R_{\theta}(\tau)}.\label{eq:zchoice}
\end{align}
Here $Z_\theta$ is a normalization constant to ensure $P_{choice}(\tau|\theta, T)$ is a probability distribution. 
Let $d_{0}$ as the probability distribution of the initial state. %
$D(\tau)=d_0(s_1)\prod_{t=1}^{L}T(s_{t+1}|s_t,a_t)$ is the probability of state transitions which leads to the trajectory $\tau$. 
The overall probability of observing trajectory $\tau$ from demonstrations hence is the product of the choice probability times the state transition probability:
\begin{align}
    P(\tau|\theta,T) = \frac{1}{Z_{\theta}}e^{R_{\theta}(\tau)}D(\tau).
\end{align}

\subsection{IRL with Multi-state Combinatorial Constraints}
Despite the success of many IRL models, many real world tasks require additional constraints to be satisfied when learning from demonstrations. 
%
In this work, we restrict ourselves to dealing with hard combinatorial constraints, as shown in Figure \ref{fig:intuition}. Note that this is not particularly restrictive since, for example, safety constraints are often hard constraints as well are constraints imposed by physical laws. Different from previous work that only defines constraints as a set of forbidden state-action pairs, which we call single-state constraints, here we consider more general  cases of multi-state combinatorial constraints. Denote $C(\tau)=\{c_i(\tau)\}$ as the set of constraints that each trajectory must satisfy, and $I_C(\tau)$ the indicator function of whether constraints $C(\tau)$ are satisfied. Formally,
\begin{align*}
    I_C(\tau)=\begin{cases}
    1, ~~~~ \text{if}~ \tau~ \text{satisfies the constraints set }~ C(\tau)\\
    0, ~~~~ \text{otherwise}
    \end{cases}
\end{align*}
We augment the MDP into the constrained MDP: $\mathcal{M}^C=\{\mathcal{S}, \mathcal{A}, T, R, C\}$. The probability of observing a trajectory $\tau$ now becomes\footnote{Notice $Z_\theta$ in Equation \ref{eq:constarined_p}  is different from $Z_\theta$ in Equation \ref{eq:zchoice} because of the introduction of $I_C$. $Z_\theta$ still normalizes the probability in Equation \ref{eq:constarined_p}. Without too much cluttering, we use the same symbol $Z_\theta$ in both equations.}
\begin{align}\label{eq:constarined_p}
    P(\tau|\theta,T)=\frac{1}{Z_{\theta}}e^{R_{\theta}(\tau)}D(\tau)I_C(\tau),
\end{align}
Given the set of expert demonstrations $\mathcal{D}$, we want to find the best reward function by minimizing the negative log likelihood function $L(\theta)$.
\begin{align*}
    \text{argmin}_{\theta}L(\theta)= \text{argmin}_{\theta}\frac{1}{|\mathcal{D}|}\sum_{\tau\in \mathcal{D}}-R_{\theta}(\tau) + \log Z_{\theta}.
\end{align*}
Notice only the terms related to the optimization variable $\theta$ are included in the rightmost equation.

%% file: tex/method.tex
\section{XOR Maximum Entropy IRL}

In this section we propose \textbf{X}OR-\textbf{M}aximum \textbf{EN}tropy Constrained Inverse Reinforcement Learning (X-MEN), to solve the inverse reinforcement learning problem with multi-state combinatorial constraints.
We develop X-MEN based on maximum entropy inverse reinforcement learning \citep{ziebart2008maximum,boularias2011relative,finn2016guided}.  Specifically, the model assumes that the expert samples the demonstrated trajectories $\{\tau_i\}$ from the distribution $P(\tau|\theta,T)$ in Equation \ref{eq:constarined_p},
where $R_{\theta}(s_t, a_t)=\theta^Tf(s_t,a_t)$ is represented by a linear combination of feature vector $f(s_t,a_t)$. $f(s_t,a_t)$ can be hand-crafted or generated by a deep neural network. 
Forward-backward dynamic programming can hardly solve this problem even if the dynamics is given due to the presence of the hard combinatorial constraints $I_C(\tau)$.
Our X-MEN has the ability to solve this problem by leveraging XOR and importance sampling to estimate  $P(\tau|\theta,T)$.
After learning with X-MEN, we can always generate valid trajectories. We use Stochastic Gradient Descent (SGD) to optimize the objective, where in each iteration we compute the gradient of the negative log likelihood:
\begin{align}\label{eq:grad_ll}
    &\nabla_{\theta}L(\theta)=\frac{1}{|\mathcal{D}|}\sum_{\tau\in \mathcal{D}}\nabla_{\theta} R_{\theta}(\tau) + \nabla_{\theta}\log Z_{\theta}\notag\\ 
    &=\frac{1}{|\mathcal{D}|}\sum_{\tau\in \mathcal{D}}\nabla_{\theta} R_{\theta}(\tau) -\sum_{\tau}P(\tau|\theta,T) \nabla_{\theta}R_{\theta}(\tau). 
\end{align}
The first term in Equation~\ref{eq:grad_ll} is $\mathbb{E}_D[\nabla_{\theta} R_{\theta}(\tau)]$ and the second term is $\mathbb{E}_P[\nabla_{\theta} R_{\theta}(\tau)]$. To compute the gradient, we estimate the second  term $\mathbb{E}_P[\nabla_{\theta} R_{\theta}(\tau)]$ using importance sampling, where the following proposal distribution  $Q(\tau|\hat{\theta})$ is used: 
\begin{align}\label{eq:Q}
Q(\tau|\hat{\theta})=\frac{1}{Z_{\hat{\theta}}}I_C(\tau)e^{-\hat{\theta}^Tf(\tau)},
\end{align}
where $f(\tau)=\sum_{t=1}^L\gamma^t f(s_t,a_t)$ is the combined feature vector of the whole trajectory.
$\hat{\theta}$ are the parameters to adjust for the proposal $Q$. While we can set other parameters, in this paper we let $\hat{\theta}=\theta$. In other words, we use the current learned parameters $\theta$ to replace $\hat{\theta}$. 
Importance sampling is needed because the probabilities of state transitions $D(\tau)$ may not have an exponential family form. Currently, the implementation of XOR-sampling samples from exponential family distributions. Nevertheless, we notice this importance sampling step can be avoided if $D(\tau)$ has an exponential family distribution. 
To approximate $\nabla_{\theta}L(\theta)$ in Equation~\ref{eq:grad_ll}, we sample $M_1$ trajectories from the dataset of demonstrations to form the set $\mathcal{D}_{M_1}$. 
Then we sample $2M_2$ trajectories from the proposal distribution $Q(\tau | \theta)$ to form two sets $\mathcal{T}_{M_2,1}^Q$, $\mathcal{T}_{M_2,2}^Q$ (each set contains $M_2$ trajectories; details on how to obtain samples from $Q$ are discussed later). 
Then, we can use $g_{\theta}$ in the following Theorem \ref{Th:compute_g} to approximate  $\nabla_{\theta}L(\theta)$:
\begin{Th}\label{Th:compute_g}
Let the model distribution $P(\tau|\theta,T)$ defined in Equation \ref{eq:constarined_p} and the gradient of the likelihood function defined in Equation \ref{eq:grad_ll}. Let $g_{\theta}$ be  
\begin{align}\label{eq:g_theta}
    g_{\theta}=\frac{1}{M_1}\sum_{\tau\in \mathcal{D}_{M_1}} f(\tau)-\frac{\sum_{\tau\in\mathcal{T}^{Q}_{M_2,1}}D(\tau)f(\tau)}{\sum_{\tau\in\mathcal{T}^Q_{M_2,2}}D(\tau)},
\end{align}
where $\mathcal{D}_{M_1}$, $\mathcal{T}_{M_2,1}^Q$, $\mathcal{T}_{M_2,2}^Q$ are defined above. We must have $g_{\theta}$ is an unbiased estimation of $\nabla_{\theta}L(\theta)$, ie., $\mathbb{E}[g_{\theta}] = \nabla_{\theta}L(\theta)$. 
\end{Th}
We leave the proof of Theorem \ref{Th:compute_g} to the  supplementary materials. 
However, in practice sampling from distribution $Q$ is intractable due to the existence of the partition function $Z_{\theta}$ and the unbiased estimation is hard to obtain. In this paper we incorporate the recently proposed XOR-Sampling to get a constant approximation of the true gradient $\nabla_{\theta}L(\theta)$.

XOR-Sampling is used to obtain samples from the proposal distribution $Q$ such that the probability of drawing a sample is sandwiched between a constant multiplicative bound of the true probability. 
XOR-Sampling is the result of a rich line of research \citep{ermon2013embed,Gomes06XORCounting,Gomes2007XORCounting}, which translates the \#-P complete sampling problem into queries to NP oracles with provable guarantees. 
The high level idea of XOR sampling is as follows. 
Suppose one would like to draw one ball uniformly at random from an urn, with access to an oracle that returns one ball from the urn once queried (implemented as an NP-oracle when sampling in a combinatorial space). Notice that the oracle will not return the balls uniformly at random; i.e., it may return the same ball every time. 
XOR-sampling removes the balls from the urn by introducing additional XOR constraints. One can prove that half of the balls are removed at random, each time when one XOR constraint is introduced. 
Hence, one keeps adding XOR constraints until there are only one ball remaining. Then the last ball is returned. 
Since the balls are removed at random, the last left  must be a random one drawn from the original set of balls.
In practice, XOR-sampling also works with weighted probability distributions. 
Our paper uses the probabilistic bound of XOR-sampling via Theorem~\ref{Th:bound2}. 
We refer the readers to \cite{ermon2013embed,fan2021xorcd,fan2021xorsgd} for the details on the discretization scheme and the choice of the parameters of XOR-sampling to obtain the bound in Theorem \ref{Th:bound2}.

\begin{Th}\label{Th:bound2}\citep{ermon2013embed}~
Let $\delta>1$, $0 < \gamma < 1$, $w: \{0,1\}^n \rightarrow \mathbb{R}^+$ be an unnormalized 
probability density function where $n=|\mathcal{S}||\mathcal{A}|$. $Q(\tau|{\theta}) \propto w(\tau)$ is the normalized distribution and $C(\tau)$ is the set of hard combinatorial constraints. Then, with probability at least $1-\gamma$, XOR-Sampling$(w, C(\tau), \delta, \gamma)$ succeeds and outputs a sample $\tau_0$ by querying $O(n\ln(\frac{n}{\gamma}))$ NP oracles. Upon success, each $\tau_0$ is produced with probability $Q'(\tau_0)$. 
Then, let $\phi:\{0,1\}^n\rightarrow\mathbb{R}^+$ be one non-negative function, then the expectation of 
one sampled $\phi(\tau)$ satisfies,
\begin{align}
    \frac{1}{\delta}\mathbb{E}_{Q}[\phi(\tau)]\leq\mathbb{E}_{Q'}[\phi(\tau)] \leq\delta\mathbb{E}_{Q}[\phi(\tau)].\label{eq:bound_eq2}
\end{align}
\end{Th}

The detailed procedure of X-MEN is shown in Algorithm \ref{alg:X-MEN}. Here we demonstrate the version of X-MEN, where the only parameter to optimize is $\theta$. 
A variant of this algorithm can be developed which back-propagate the gradient over the feature vector $f(s,a)$ as well, if $f(s,a)$ is represented as a neural network and also can be updated during learning.
X-MEN takes as inputs the feature vector $f(s,a)$, transition probability $D(\tau)$, constraint set $C(\tau)$, training data $\{\tau_i\}_{i=1}^N$, initial model parameter $\theta_0$, 
the learning rate $\eta$, the number of SGD iterations $K$,
XOR-Sampling parameters $(\delta, \gamma)$, and batch sizes $M_1$, $M_2$, and outputs the averaged learned parameter $\overline{\theta_{K}}$. 
To approximate $\mathbb{E}_{P}[\nabla_{\theta}R_{\theta}(\tau)]$ at the $k$-th iteration, X-MEN draws $2M_2$ samples $\tau'_1, \dots, \tau'_{2M_2}$ from the proposal distribution $Q(\tau|{\theta})$ using XOR-Sampling, where $M_2$ is a user-determined sample size.
Because XOR-Sampling has a failure rate, X-MEN repeatedly call XOR-Sampling until all $2M_2$ samples are obtained successfully (line 3 -- 8). Then, X-MEN also draws $M_1$ samples from the training set $\{\tau_i\}_{i=1}^N$ uniformly at random to approximate $\mathbb{E}_{\mathcal{D}}[\nabla_{\theta}R(\theta)]$. 
Once all the samples are obtained, X-MEN uses 
$g_k = \frac{1}{M_1}\sum_{\tau\in \mathcal{D}_{M_1}} f(\tau)-\frac{\sum_{j=1}^{M_2}D(\tau'_j)f(\tau'_j)}{\sum_{j=M_2+1}^{2M_2}D(\tau'_j)}$ as an approximation for the gradient of the negative log likelihood. 
$\theta$ is updated following the rule $\theta_{k+1} = \theta_{k}-\eta g_k$ for  $K$ steps, where $\eta$ is the learning rate. Finally, the average of $\theta_1, \ldots, \theta_K$, namely  $\overline{\theta_{K}}=\frac{1}{K}\sum_{k=1}^{K}\theta_k$ is the output of the algorithm.
We show in the next sections that X-MEN enjoys the property of convergence to the global optimum of the objective in linear number of iterations, and illustrate how to incorporate XOR-Sampling into our framework for sample generation with strict constraint satisfaction.

\begin{algorithm}[t!]
   \caption{XOR Maximum Entropy Constrained Inverse Reinforcement Learning (X-MEN)}
   \label{alg:X-MEN}
   \LinesNumbered
   \KwIn{$\theta_0, f(s, a), K, \eta, \delta, \gamma, D(\tau), C(\tau), M_1,M_2,\mathcal{D}$.}
   \For{$k=0$ {\bfseries to} $K$}{
        $j\gets 1$     \tcp*[f]{\text{$M_1$ and $M_2$ are batch size}}\\
        \While{$j\leq 2M_2$}{
            $\tau'\gets$ XOR-Sampling$\left(e^{-{\theta_k}^T f(\tau)}, C(\tau), \delta, \gamma\right)$ 
            \If{$\tau' \neq Failure$}  {
                $\tau'_j\gets \tau'$; $j\gets j+1$ 
            }
      }
      Get samples~~ $\mathcal{D}_{M_1}=\{\tau_j\}_{j=1}^{M_1}$ from $\mathcal{D}$.\\
      $g_k=\frac{1}{M_1}\sum_{\tau\in \mathcal{D}_{M_1}} f(\tau)-\frac{\sum_{j=1}^{M_2}D(\tau'_j)f(\tau'_j)}{\sum_{j=M_2+1}^{2M_2}D(\tau'_j)}$\\
      Update the parameters~~ $\theta_{k+1} = \theta_{k}-\eta g_k$
    }
    \textbf{return}$~~\overline{\theta_{K}}=\frac{1}{K}\sum_{k=1}^{K}\theta_k$
\end{algorithm}

\subsection{Linearly Converge to the Global Optimum}
Suppose the only parameter to learn is $\theta$, in other words, $f(x,a)$ are fixed, 
the reward function $R_{\theta}(\tau)$ is represented by a linear combination of hand-crafted features), 
we can easily see that the objective is convex with regard to $\theta$. Under this circumstance, we show that X-MEN converges to the global optimum of the log likelihood function in addition to a vanishing term. 
Moreover, the speed of the convergence is linear 
with respect to the number of stochastic gradient descent steps. 
Denote $Var_{\mathcal{D}}(f(\tau)) = \mathbb{E}_{\mathcal{D}}[||f(\tau)||_2^2] - ||\mathbb{E}_{\mathcal{D}}[f(\tau)]||_2^2$ and 
$Var_{P}(f(\tau)) = \mathbb{E}_{{P}}[||f(\tau)||_2^2] - ||\mathbb{E}_{{P}}[f(\tau)]||_2^2$ as the total variations of $f(\tau)$ w.r.t. the data distribution $P_{\mathcal{D}}$ and model distribution $P(\tau|\theta,T)$.
The precise mathematical theorem states:
\begin{Th}\label{Th:main}
(main)~ Let $P(\tau|\theta,T)$ and $Q(\tau|{\theta})$ as defined in Equation \ref{eq:constarined_p} and \ref{eq:Q},  $R_{\theta}(\tau)=\theta^Tf(\tau)$. 
Given trajectories $\mathcal{D}=\{\tau_i\}_{i=1}^N$ and the objective function $L(\theta)$,  denote $OPT=\min_{\theta} L(\theta)$ and $\theta^*=\text{argmin}_{\theta}L(\theta)$. 
Let $Var_{\mathcal{D}}(f(\tau))\leq\sigma_1^2$ and $\max_{\theta}Var_{P}(f(\tau))\leq \sigma_2^2$. 
Suppose $1\leq\delta\leq\sqrt{2}$ is used in XOR-sampling, the learning rate $\eta\leq \frac{2-\delta^2}{\sigma_2^2\delta}$, and $\overline{\theta_K}$ is the output of X-MEN. We have: 
\begin{align*}
     \mathbb{E}[L(\overline{\theta_K})]-OPT \leq\frac{\delta^2||\theta_0-\theta^*||_2^2}{2\eta K}+\frac{\eta\sigma_1^2}{\delta^2 M_1}+\frac{\eta\sigma_2^2}{\delta^2M_2}.
\end{align*}
\end{Th}

X-MEN is the first provable algorithm which converges to the global optimum of the likelihood function and a tail term for constrained inverse reinforcement learning problems. Moreover, the rate of the  convergence is linear in the number of SGD iterations $K$. Previous approaches for IRL problems with hard combinatorial constraints do not have such tight bounds. 

The main challenge to prove Theorem~\ref{Th:main}
lies in the fact that we cannot ensure the unbiasedness of the 
gradient estimator. 
Because the objective is convex with respect to $\theta$ and smooth, a gradient descent algorithm can be proven 
to be linearly convergent towards the optimal value if the 
expectation of the estimated gradient is unbiased, ie, $\mathbb{E}[g_k] = \nabla_{\theta} L(\theta_k)$. 
However, even though we apply XOR-sampling, which has
a constant approximation bound in generating 
samples from the model distribution, 
we still cannot guarantee the unbiasedness of $g_k$. 
Instead, using the constant factor approximation of XOR-Sampling, which is formally stated in Theorem~\ref{Th:bound2}, the bound
for $g_{k}$ is in the following form:
\begin{align}
    \frac{1}{\delta^2} [\nabla L(\theta_k)]^+ \leq \mathbb{E}[g_k^+]
    \leq \delta^2 [\nabla L(\theta_k)]^+,\label{eq:gb1}\\
    \delta^2 [\nabla L(\theta_k)]^- \leq \mathbb{E}[g_k^-]
    \leq \frac{1}{\delta^2} [\nabla L(\theta_k)]^-.\label{eq:gb2}
\end{align}
Here, $\delta>1$ is a constant factor, $[f]^+$ means the positive part of $f$, ie,
$[f]^+ = \max\{f, \mathbf{0}\}$, and $[f]^-$ means the negative part of $f$, ie,
$[f]^- = \min\{f, \mathbf{0}\}$. 
The bound in Equation~\ref{eq:gb1} and \ref{eq:gb2}
can be proven by bounding the nominator and the denominator of Equation \ref{eq:g_theta}, and we leave the proof in the supplementary materials.  

The proof of Theorem \ref{Th:main} relies mainly on the following Theorem \ref{Th:XOR_SGD_bound} which bounds the errors of Stochastic Gradient
Descent (SGD) algorithms which only have
access to constant approximate gradient vectors. 
Theorem~\ref{Th:XOR_SGD_bound} was proved in \cite{fan2021xorcd}, to help bound the errors of learning an exponential family model.
Theorem \ref{Th:XOR_SGD_bound} requires function 
$f$ to be $L$-smooth. $f(\theta)$ is $L$-smooth if and only if $||f(\theta_1) - f(\theta_2)||_2 \leq L ||\theta_1 - \theta_2||_2$. 
Notice that the conditions of Theorem \ref{Th:main}
automatically guarantee the $L$-smoothness of the objective and we leave the proof in the appendix.

\begin{Th}\label{Th:XOR_SGD_bound}\citep{fan2021xorcd}~
Let $f:\mathbb{R}^d\rightarrow \mathbb{R}$ be a $L$-smooth convex function and $\theta^*=\text{argmin}_{\theta} f(\theta)$. In iteration $k$ of SGD, $g_k$ is the estimated gradient, i.e., $\theta_{k+1}=\theta_{k}-\eta g_k$. If $Var(g_k)\leq \sigma^2$, and there exists $1\leq c\leq\sqrt{2}$ s.t. $\frac{1}{c}[\nabla f(\theta_k)]^+ \leq \mathbb{E}[g_k^+]\leq c[\nabla f(\theta_k)]^+$ and $c[\nabla f(\theta_k)]^- \leq \mathbb{E}[g_k^-]\leq \frac{1}{c}[\nabla f(\theta_k)]^-$, then for any $K>1$ and step size $\eta\leq \frac{2-c^2}{Lc}$, let $\overline{\theta_K}=\frac{1}{K}\sum_{k=1}^K \theta_k$, we have 
\begin{align}\label{eq:XOR_SGD_bound}
    \mathbb{E}[f(\overline{\theta_K})]-f(\theta^*)\leq \frac{c||\theta_0-\theta^*||_2^2}{2\eta K}+\frac{\eta\sigma^2}{c}.
\end{align}
\end{Th}

The proof of Theorem \ref{Th:main} is to apply 
Theorem \ref{Th:XOR_SGD_bound} on the objective $L(\theta)$ and noticing that $L(\theta)$ is $L$-smooth when the total variation $Var_P(f(\tau))$ is bounded \citep{fan2021xorcd}.
Theorem~\ref{Th:main} states that in expectation, the difference between the output of X-MEN and the true optimum $OPT$ is bounded by a term that is inversely proportional to the number of iterations $K$ and a tail term $\frac{\eta\sigma_1^2}{\delta^2 M_1}+\frac{\eta\sigma_2^2}{\delta^2M_2}$.
To reduce the tail term with fixed steps $\eta$, we can generate more samples at each iteration to reduce the variance (increase $M_1$ and $M_2$). 
In addition, to quantify the computational complexity of X-MEN, 
we prove the following theorem in the supplementary materials detailing the number of queries to NP oracles needed for X-MEN.
\begin{Th}\label{Th:num_queries}
Let $|\mathcal{S}|$ and $|\mathcal{A}$ be the number of all possible states and all possible actions, respectively, then X-MEN in Algorithm \ref{alg:X-MEN} uses $O\left(K|\mathcal{S}||\mathcal{A}|\ln\frac{|\mathcal{S}||\mathcal{A}|}{\gamma}+KM_2\right)$ queries to NP oracles.
\end{Th}

%% file: tex/related.tex
\section{RELATED WORK}
Max-Ent IRL models were first proposed \citep{ziebart2008maximum} to addresses the inherent ambiguity of possible reward functions and induced policies for an observed behavior, during the training of which a forward-backward dynamic programming algorithm were used to exactly compute the partition function and marginal probability \citep{snoswell2020revisiting}, assuming the knowledge of the transition probability. Relative Entropy IRL \citep{boularias2011relative} extends this work by leveraging an importance sampling approach to estimate the partition function unbiasedly without knowing the dynamics. Guided Cost Learning \citep{finn2016guided} further learns a Max-Ent model with policy optimization. Later work accommodates arbitrary nonlinear reward functions such as neural networks \citep{finn2016guided,kalweit2020deep,wulfmeier2015maximum}, instead of a linear combination of features. Recently proposed Generative Adversarial Imitation Learning (GAIL) \citep{ho2016generative} is an imitation learning method that does not require estimating likelihoods.
However, while Markovian rewards do often provide a succinct and expressive way to specify the objectives of a task, they cannot capture all possible task specifications, especially additional constraints \citep{vazquez2017learning}. Recent work on constrained IRL only focuses on local constraints of states, actions and features \citep{chou2018learning,subramani2018inferring,mcpherson2018modeling}, which can hardly represent all the real world scenarios as most constraints are trajectory long. Other methods focus on learning constraints from the demonstrations, such as maximum likelihood constraint inference \citep{scobee2019maximum,kalweit2020deep,anwar2020inverse,mcpherson2021maximum}. Our approach differs from all the existing methods and addresses the open question of learning with hard combinatorial constraints. We adapt the Max-Ent framework to allow us to reason about all the trajectories that satisfy the constraints during the contrastive learning process. Here we only consider pre-defined constraints. One should notice that even with the full knowledge of transition probability, dynamic programming cannot work well under trajectory-long constraints since it has no knowledge of any hard combinatorial information.
X-MEN was motivated by the recent proposed probabilistic inference via hashing and randomization technique for both sampling \citep{ermon2013embed,ivrii2015computing}, counting \citep{Gomes2006NearUniformSampling,ding2019towards}, and marginal inference  problems \citep{Ermon13Wish,kuck2019adaptive,Chakraborty2014DistributionAwareSA,Chakraborty2015WeightedCounting,belle2015hashing} with constant approximation guarantees. Latest work also show the success of XOR-Sampling \citep{ermon2013embed} to boost stochastic optimization algorithms \citep{fan2021xorsgd} and improve machine learning tasks on structure generation \citep{fan2021xorcd}.

%% file: tex/exp.tex
\section{EXPERIMENTS}

\begin{figure*}[t]
\centering
\subfigure[reward map]{\label{fig:reward}
\includegraphics[width=0.3\linewidth]{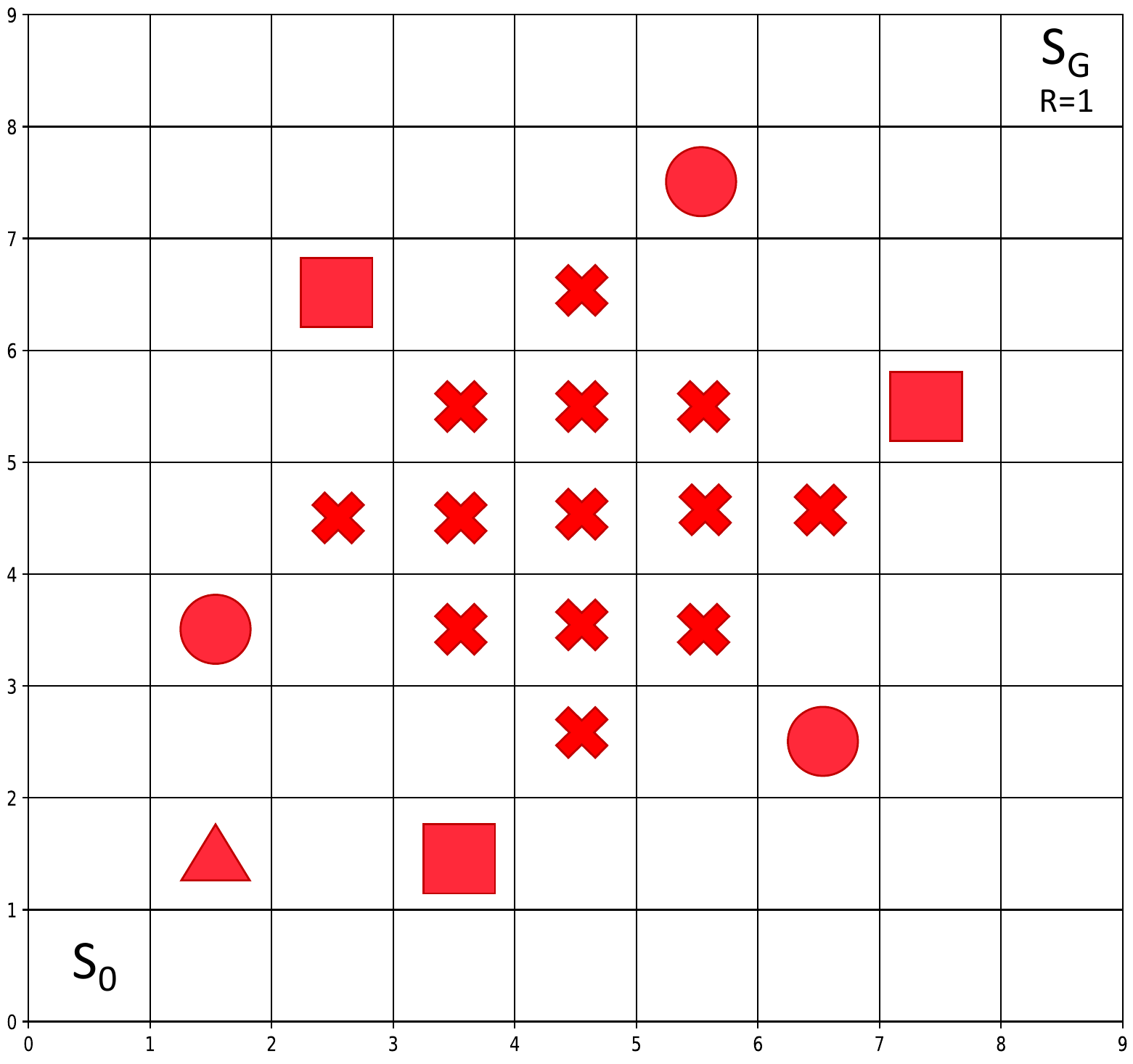}}
\subfigure[ground truth]{\label{fig:gt}
\includegraphics[width=0.3\linewidth]{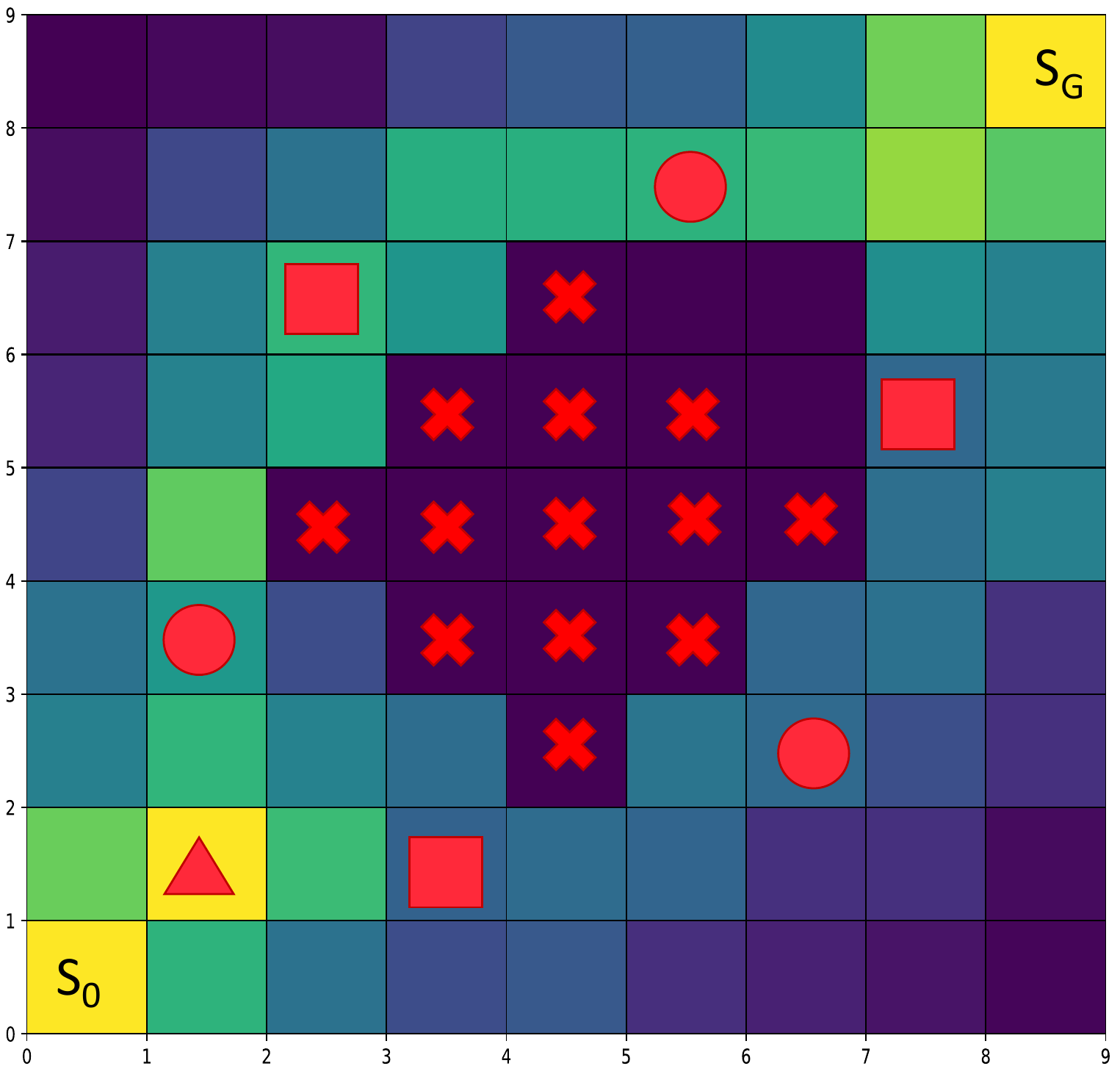}}
\subfigure[X-MEN]{\label{fig:x-men}
\includegraphics[width=0.35\linewidth]{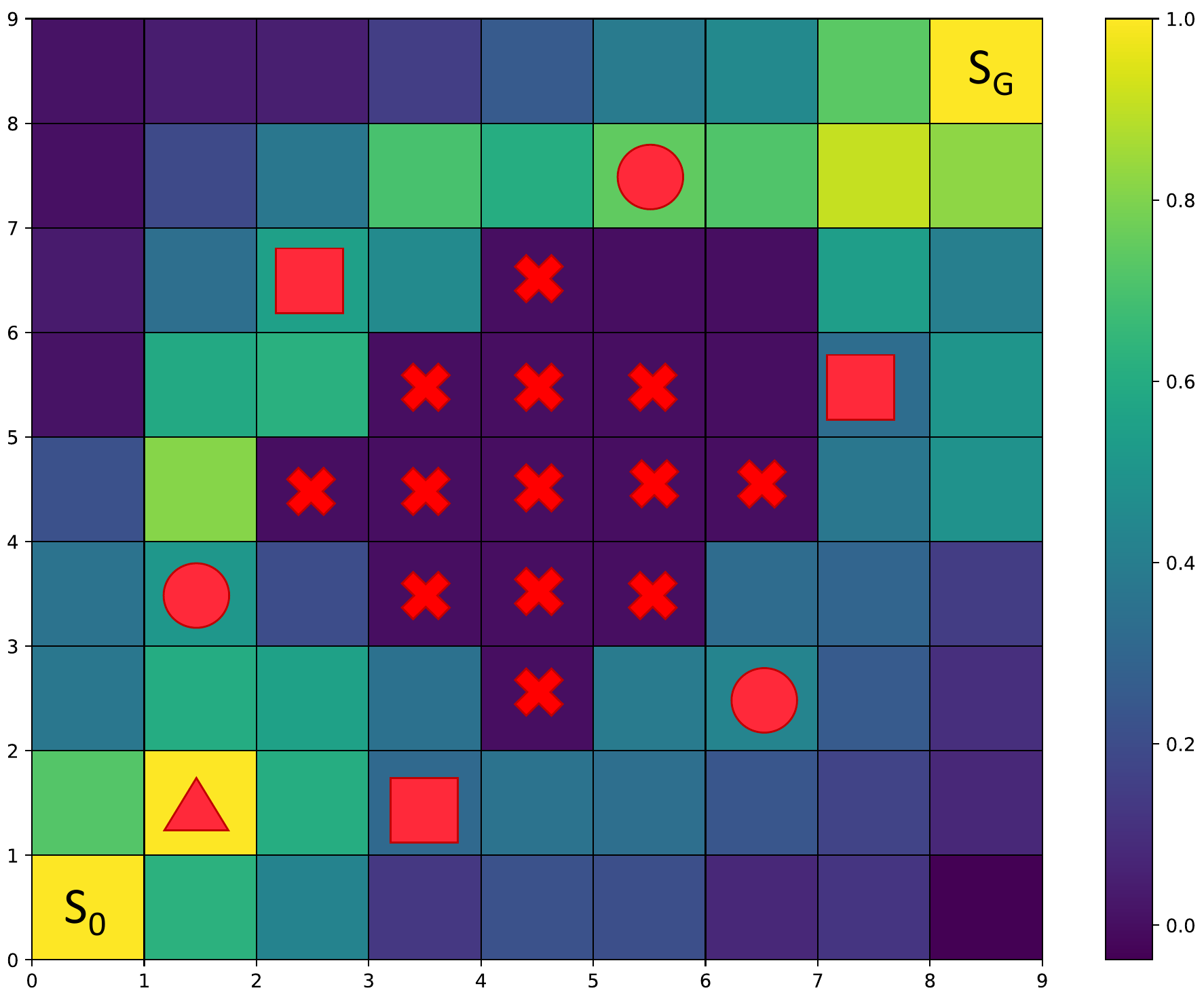}}
\subfigure[Maxent]{\label{fig:maxent}
\includegraphics[width=0.3\linewidth]{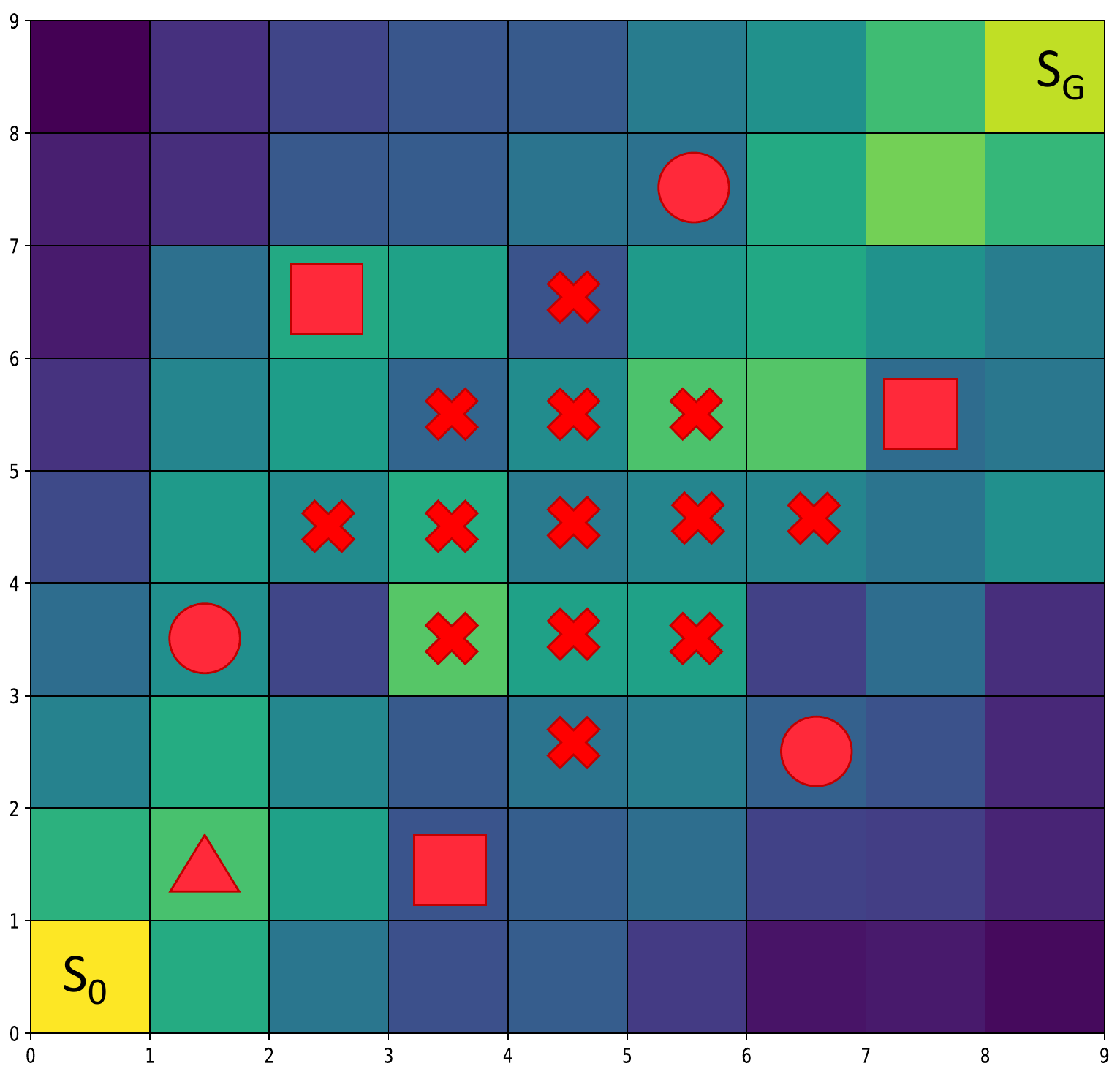}}
\subfigure[RE-IRL]{\label{fig:re}
\includegraphics[width=0.3\linewidth]{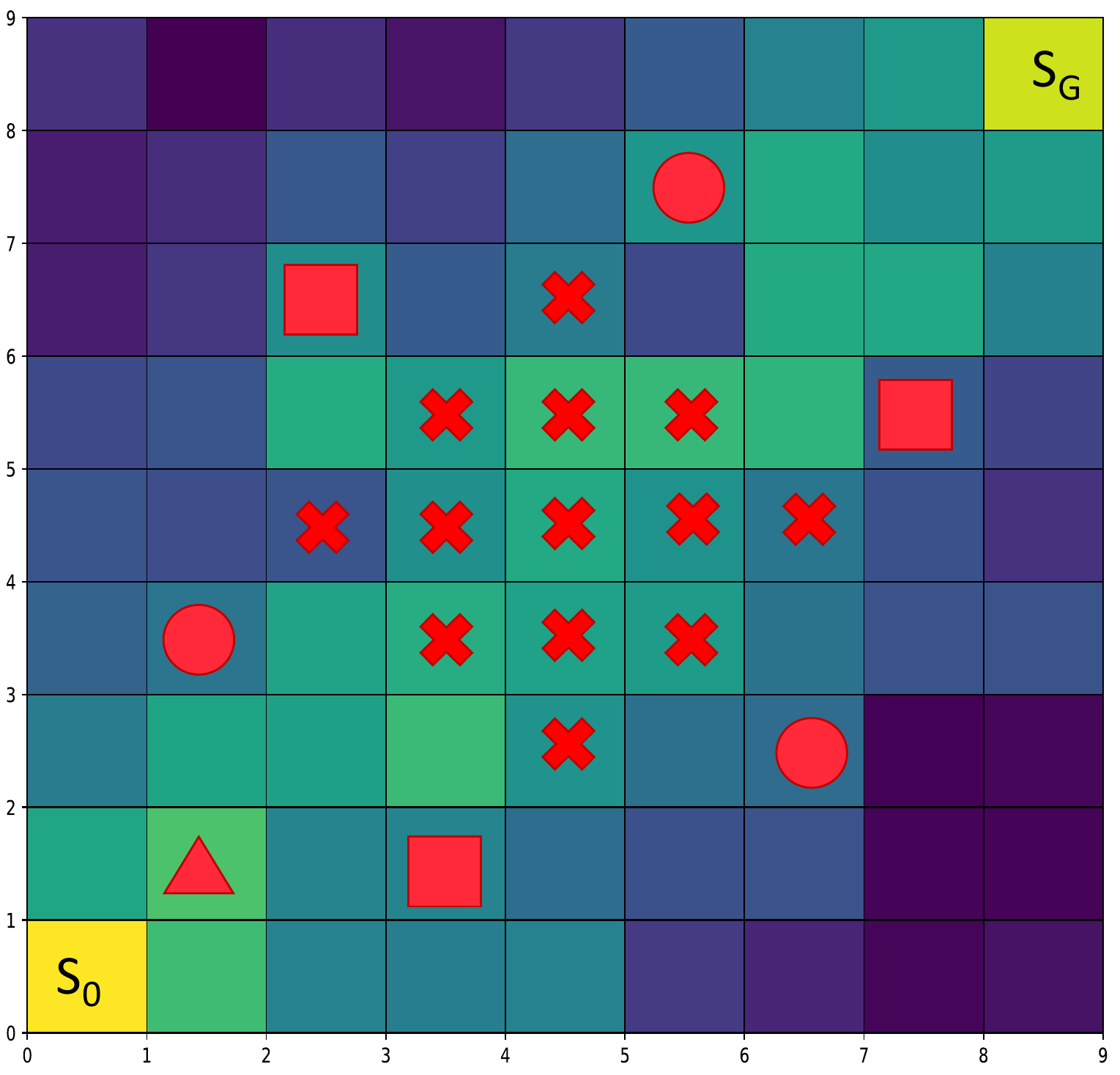}}
\subfigure[MLCI]{\label{fig:mlci}
\includegraphics[width=0.35\linewidth]{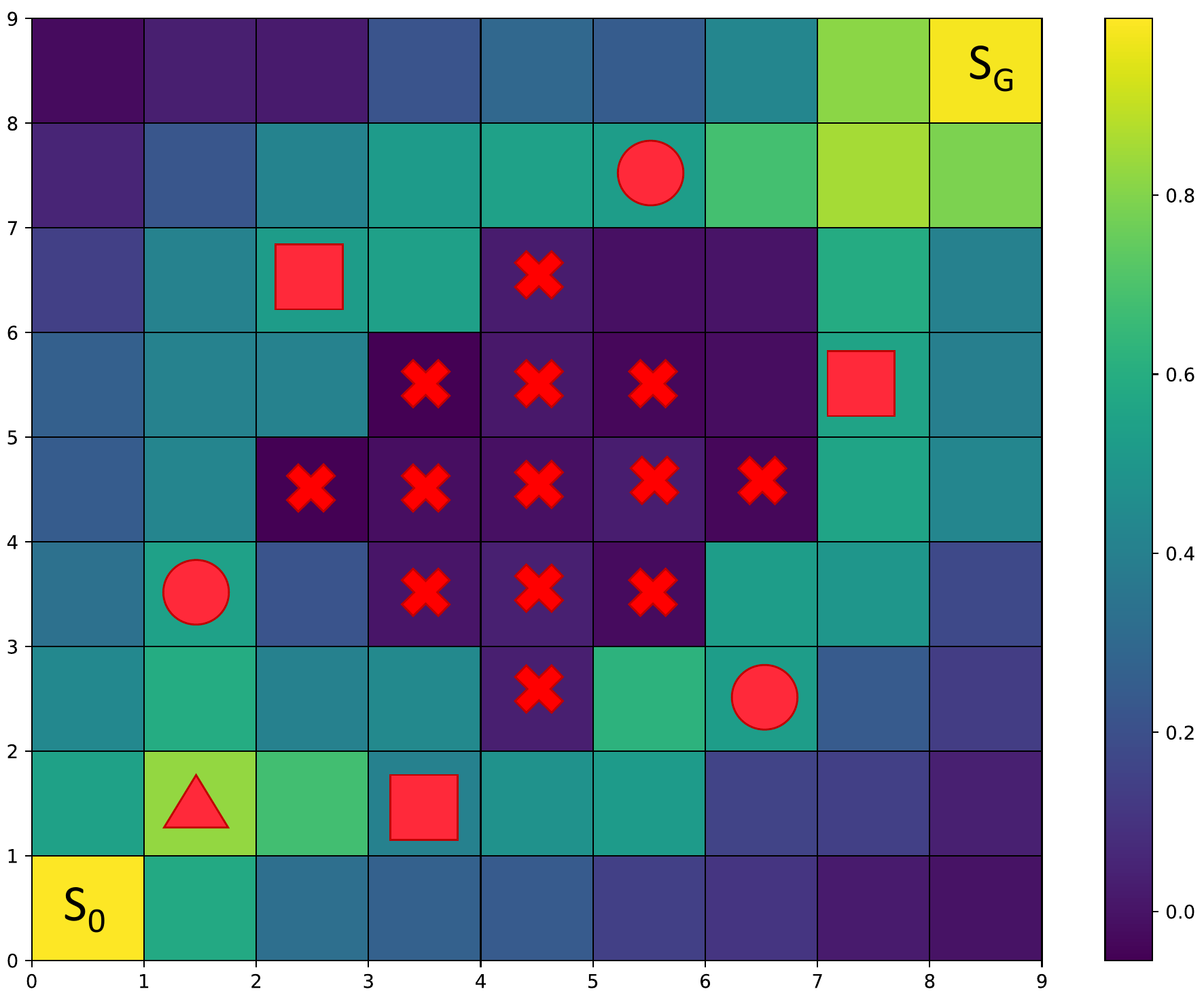}}
\caption{The superior performance of X-MEN against baselines in the grid world environment.\textbf{(a)} The ground truth reward map of the $9\times9$ gridworld. The reward of each state is 0, except for $S_G$ which is 1. Red symbols denotes constraints, where the red triangle denotes the state that must be passed through first among all the symbols, red crosses denote the states that can never be passed through, and the agent must pass through only one red square and one red circle. \textbf{(b)-(e)} The marginal probability of passing through each state of the ground truth demonstration and the distribution generated by different learning algorithms. We can see distribution of trajectories from X-MEN matches with the demonstration the most. Neither Maxent IRL nor RE-IRL can handle constraints. While MLCI knows "where not to go", it has difficulty in knowing "where must go" and we show in Figure \ref{fig:structure} that it can not generate $100\%$  trajectories satisfying constraints.} 
\label{fig:gridworld}
\end{figure*}

\begin{figure*}[t]
\subfigure[]{\label{fig:sample2sample}
\includegraphics[width=0.32\linewidth]{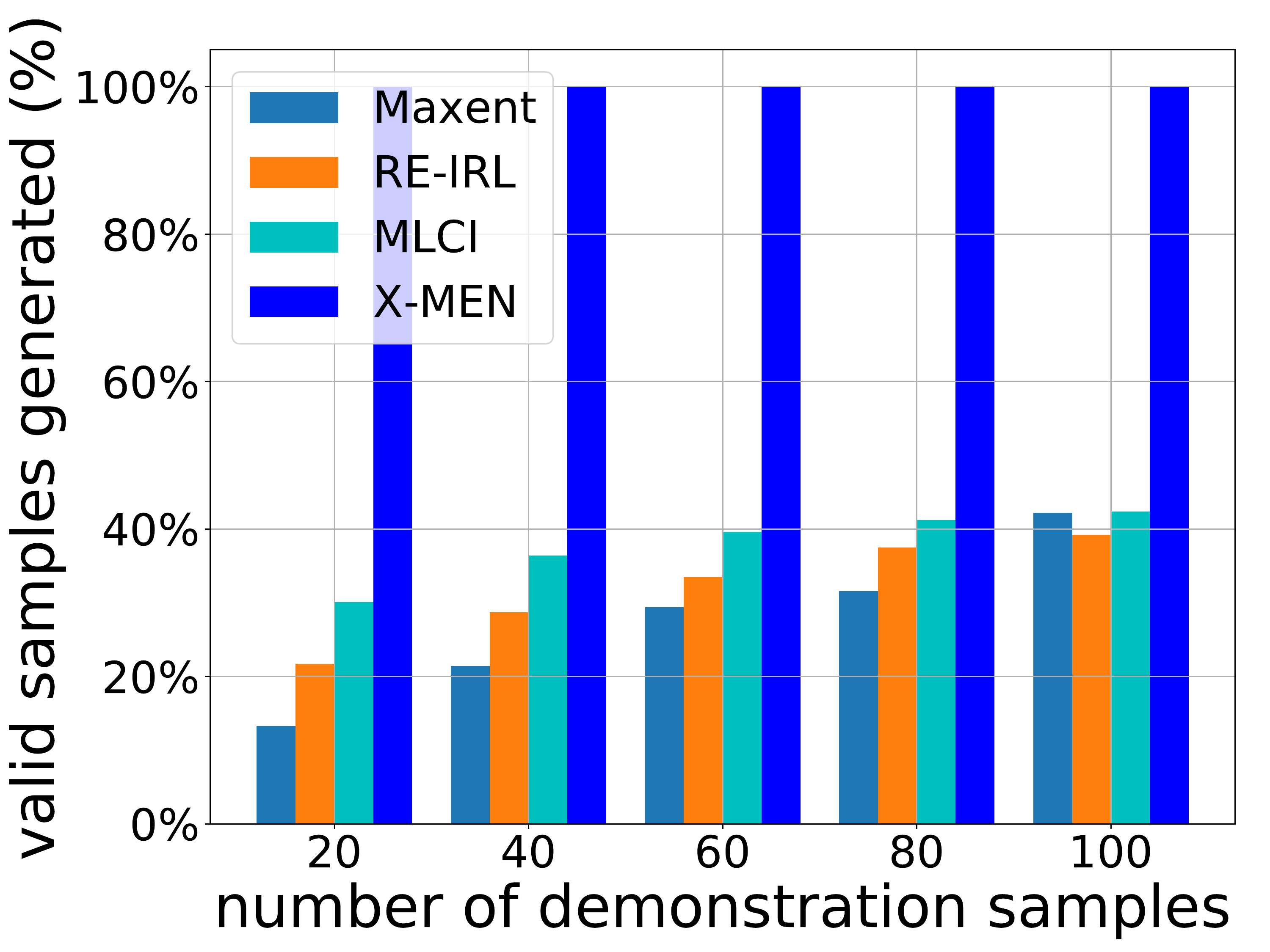}}
\subfigure[]{\label{fig:sample2time}
\includegraphics[width=0.32\linewidth]{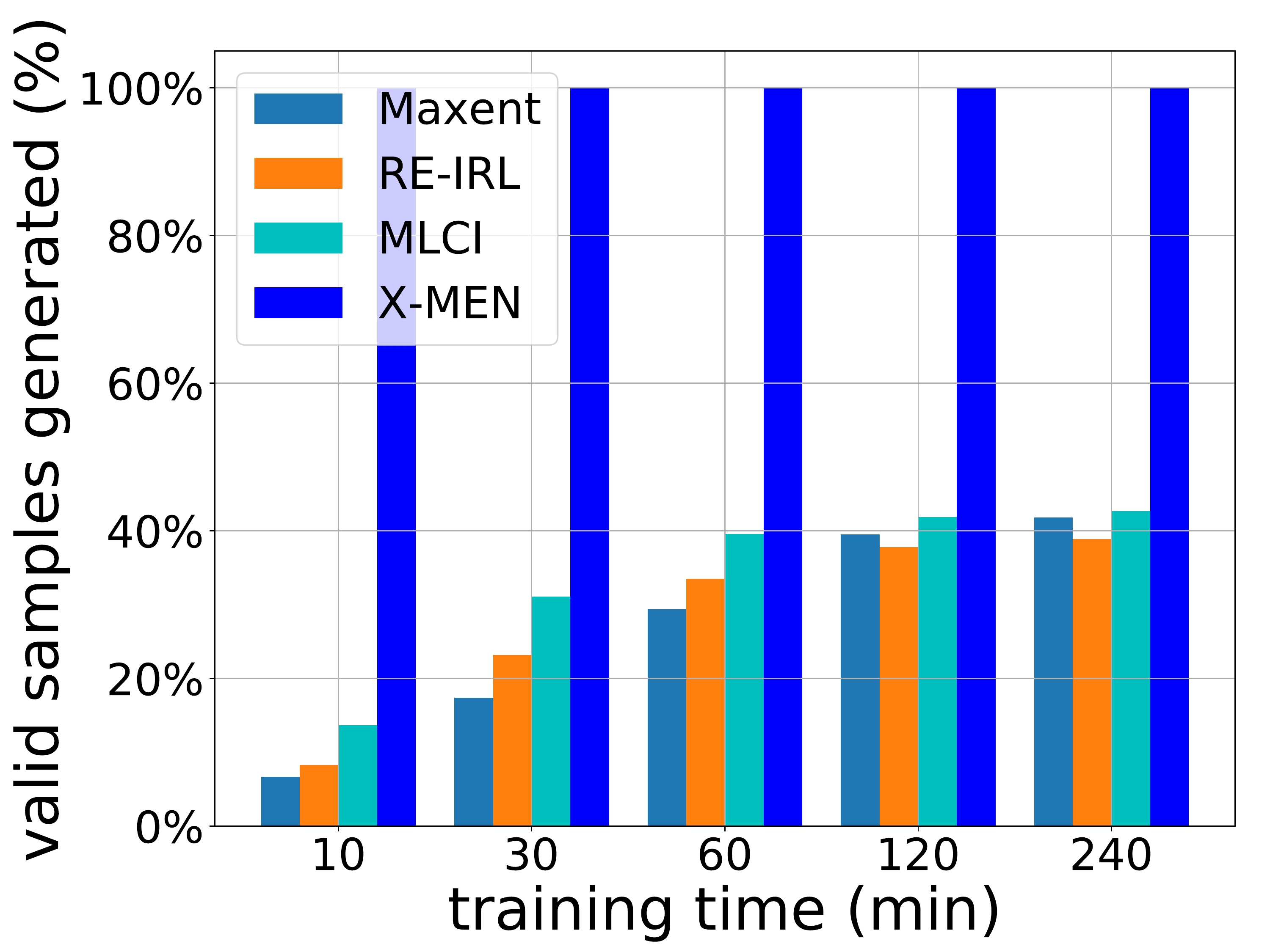}}
\subfigure[]{\label{fig:distribution}
\includegraphics[width=0.32\linewidth]{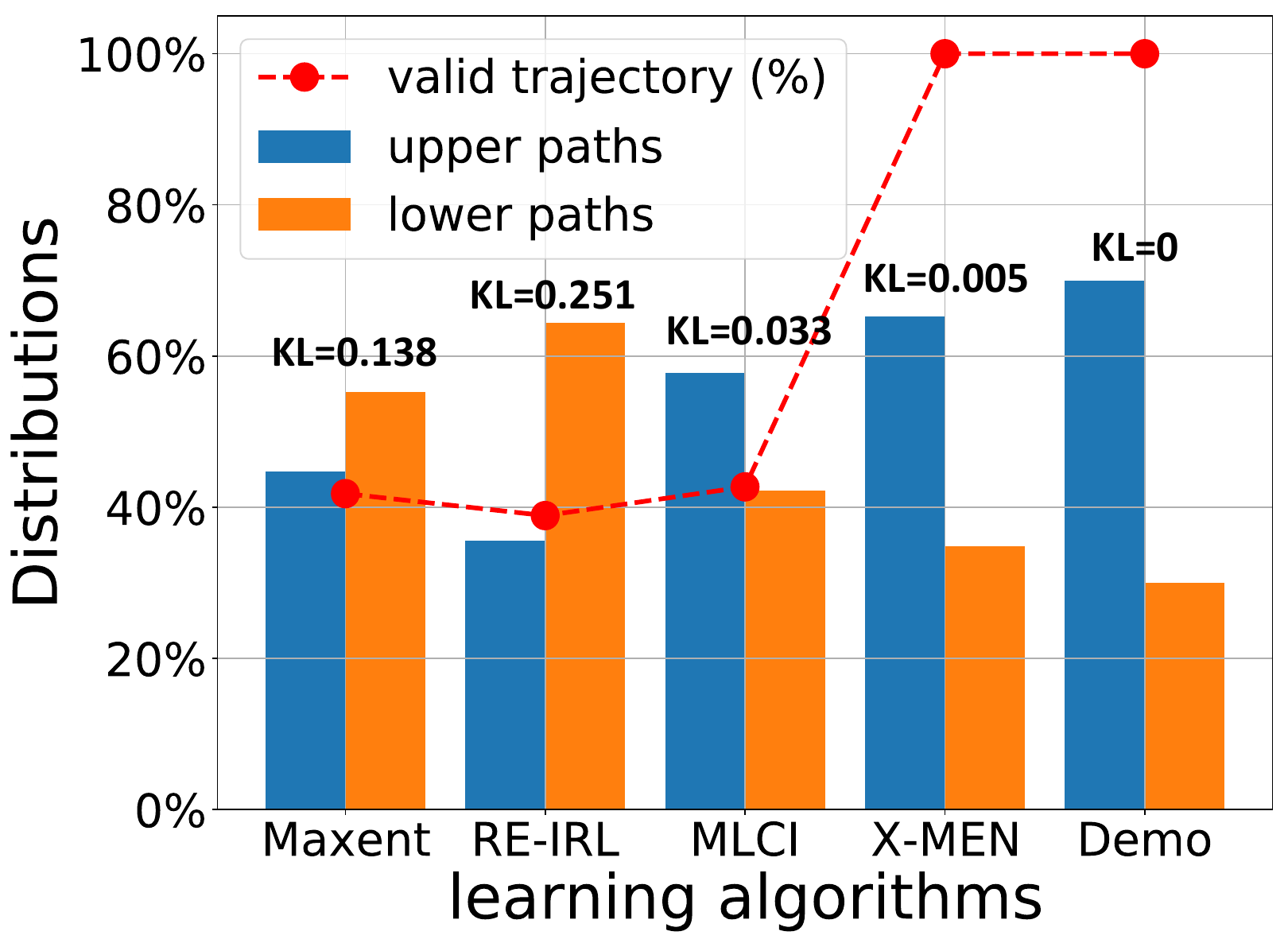}}
\caption{X-MEN outperforms competing approaches by producing 100\% valid trajectories while capturing the inductive bias in demonstration on a $9\times 9$ gridworld benchmark shown in Figure \ref{fig:reward}. (\textbf{Left}) The percentage of valid trajectories generated by different algorithms, varying the number of demonstration trajectories.  (\textbf{Middle}) The percentage of valid trajectories generated by different algorithms varying training time. (\textbf{Right}) The dashed line shows the percentage of valid trajectories generated from different algorithms. The bars show the distributions of these valid trajectories grouped by different types of paths (upper paths or lower paths). X-MEN generates 100$\%$ valid trajectories, and has the minimal KL divergence $0.005$ towards that of demonstration.} 
\label{fig:structure}
\end{figure*}

We conduct the experiments similar to those in \cite{scobee2019maximum}, where we first show the superior performance of X-MEN in a synthetic grid world benchmark and then analyze trajectories from humans as they navigate around obstacles and follow certain constraints on the floor. To obtain final trajectories, X-MEN first draws trajectories from the proposal distribution, and then re-samples from this trajectory pool according to the importance weights. 
For comparison, we compare with classic Max-Ent IRL \citep{ziebart2008maximum}, RE-IRL \citep{boularias2011relative} and recently proposed maximum likelihood constraint inference (MLCI) \cite{scobee2019maximum} which can mask out the "not to go" states in the transition distribution. We implement X-MEN using IBM ILOG CPLEX Optimizer 12.63 for queries to NP oracles and XOR-Sampling parameters are same as \cite{fan2021xorcd}. Experiments are carried out on a cluster, where each node has 24 cores and 96GB memory.



\subsection{Grid World}
We consider a 9×9 grid world. The state corresponds to the location of the agent on the grid. The agent has three actions for moving up, right, or diagonally to the upper right by one cell. The objective is to move from the starting state in the bottom-left corner $s_0$ to the goal state in the up-right corner $s_G$. ). Every state-action pair produces a distance feature, and the cumulative reward is inverse proportional to distance, which encourages short trajectories. There are additionally three more types of constraints, denoted as red symbols shown in Figure \ref{fig:reward}. The red triangle denotes the state that must be passed through first among all the symbols, red crosses denote the states that can never be passed through, and the agent must pass through only one red square and one red circle. The demonstration trajectories satisfies all the constraints and have an inductive bias: $70\%$ trajectories move along the upper paths and $30\%$ move along the lower path.

Due to the presence of hard constraints, recovering the reward map cannot be considered as the sole  performance metric for a learning algorithm.
In fact, an IRL agent with the groundtruth reward map may produce sub-optimal actions if he violates constraints. 
Therefore, we show in Figure \ref{fig:gt}-\ref{fig:mlci} the marginal distributions of passing each grid cell generated by aggregating 100 trajectories produced by different learning algorithms and the groundtruth demonstrations. We can see distribution of trajectories from X-MEN matches the demonstrations the most. Neither Maxent IRL nor RE-IRL can handle constraints. While MLCI knows "where not to go", it has difficulty in knowing "where must go" as the probability of the state marked as triangle is not 1 (we constrain that the agent must go through the triangle). Figure \ref{fig:structure} further computes the percentage of valid trajectories generated by different algorithms varying the number of demonstration trajectories (\ref{fig:sample2sample}) and training time (\ref{fig:sample2time}). X-MEN always generates $100\%$ valid trajectories while the competing methods satisfy no more than $50\%$. Moreover, we can see from the trend that even we keep increasing the number of demonstrations and the training time, the increase in baseline performance is minimal.
Figure \ref{fig:distribution} compares the recovered distribution of the trajectories, where we can see X-MEN has the minimal KL divergence 0.005 towards the ground truth distribution of demonstration.

\subsection{Human Obstacle Avoidance}
In our second example, we analyze trajectories from humans as they navigate around obstacles on the floor and follow certain constraints. We map these continuous trajectories into trajectories through a grid world where each cell represents a a 1ft-by-1ft area on the ground. The state corresponds to the location of the agent on the grid. The human agents are attempting to reach a fixed goal state $S_G$ from a given initial state $S_0$, as shown in Figure 1 in the appendix. The agent has only two actions for moving up or moving right. The shaded regions represent obstacles in the human’s environment that cannot be passed through, and the red circle represent a "must pass" choke point that every person has to walk through. Additional hard constraints are that human cannot take the same action consecutively for more than 3 times.

Demonstrations were collected from 10 volunteers, who want to move from the start state to the goal state without violating any constraints. 
Empirical observations reveal that volunteers tend to follow the shortest paths given these constraints. We train both our model and the competing approaches using these demonstrations within the same training time of $4$ hours and use 16 trajectory samples in each SGD iteration. Generated trajectories are shown in Figure 1 in the appendix, where we can see X-MEN is able to successfully avoid obstacles and pass the "must go" choke point, and the 10 generated trajectories shown in the figure are indeed the shortest paths from the start state to the goal (matching human demonstrations). Competing approaches do not generate trajectories that satisfy constraints, while the trajectories generated by X-MEN are $100\%$ valid.



%% file: tex/conclusion.tex
\section{CONCLUSION}
We proposed X-MEN, a novel XOR maximum entropy framework for constrained Inverse Reinforcement Learning.
We showed theoretically that X-MEN converges in linear speed towards the global optimum of the likelihood  function for solving IRL problems. Empirically, we demonstrated the superior performance of X-MEN on two navigation tasks with additional hard combinatorial constraints. 
In all tasks, X-MEN generates 100\% valid samples and the generated trajectories closely match the distribution of the training set. 
%
%
For future work, we would like to extend X-MEN to model-free reinforcement learning while preserving the  theoretical guarantees. We also intend to test richer representations of the reward function in form of  deep networks on real-world, large-scale constrained IRL tasks.

%% file: tex/appendix.tex
\section*{Appendix}

\section{Proof of Theorem 1}\label{app:Th:compute_g}

\begin{proof}(Theorem 1)
Formally, denote $g_Z$ to estimate $\nabla_{\theta}\log Z_{\theta}$ as follows using importance sampling:
\begin{align}
    g_Z &=\frac{1}{M_2} \sum_{\tau\in\mathcal{T}^{Q}_{M_2,1}}\frac{P(\tau|\theta,T)}{Q(\tau|\hat{\theta})}\nabla_{\theta}R_{\theta}(\tau)\\
    &=\frac{1}{M_2}\sum_{\tau\in\mathcal{T}^{Q}_{M_2,1}}\frac{Z_{\hat{\theta}}}{Z_{\theta}}\frac{D(\tau)e^{-R_{\theta}(\tau)}} {e^{-\hat{\theta}^Tf(\tau)}}\nabla_{\theta}R_{\theta}(\tau)
\end{align}
where $M_2$ is the sample size of importance distribution $Q$. $\mathcal{T}$ denotes the set of all feasible trajectories, and $\mathcal{T}_M^Q$ denotes the set of $M$ trajectories sampled from distribution $Q$. 
Then, we can write $Z_{\theta}$ out and further use importance sampling 
\begin{align}
    g_Z&=\frac{\frac{Z_{\hat{\theta}}}{M_2}\sum_{\tau\in\mathcal{T}^{Q}_{M_2,1}}\frac{D(\tau)e^{-R_{\theta}(\tau)}} {e^{-\hat{\theta}^Tf(\tau)}}\nabla_{\theta}R_{\theta}(\tau)}{\sum_{\tau\in\mathcal{T}} e^{-R_{\theta}(\tau)}D(\tau)I_C(\tau) }\\
    &=\frac{\frac{Z_{\hat{\theta}}}{M_2}\sum_{\tau\in\mathcal{T}^{Q}_{M_2,1}}\frac{D(\tau)e^{-R_{\theta}(\tau)}} {e^{-\hat{\theta}^Tf(\tau)}}\nabla_{\theta}R_{\theta}(\tau)}{\frac{1}{M_2}\sum_{\tau\in\mathcal{T}^Q_{M_2,2}} \frac{e^{-R_{\theta}(\tau)}D(\tau)I_C(\tau)}{I_C(\tau)e^{-\hat{\theta}^Tf(\tau)}}Z_{\hat{\theta}} }
\end{align}
where we can see the term $Z_{\hat{\theta}}/M_2$ in both the nominator and the denominator are cancelled out. Therefore, by simplifying this equation we get
\begin{align}
    g_Z=\frac{\sum_{\tau\in\mathcal{T}^{Q}_{M_2,1}}\frac{D(\tau)e^{-R_{\theta}(\tau)}} {e^{-\hat{\theta}^Tf(\tau)}}\nabla_{\theta}R_{\theta}(\tau)}{\sum_{\tau\in\mathcal{T}^Q_{M_2,2}} \frac{e^{-R_{\theta}(\tau)}D(\tau)}{e^{-\hat{\theta}^Tf(\tau)}}}
\end{align}
Especially, if we represent $R_{\theta}(\tau)$ as linear combination of features and let $\hat{\theta}=\theta$, this equation further simplifies to 
\begin{align}
     g_Z =\frac{\sum_{\tau\in\mathcal{T}^{Q}_{M_2,1}}D(\tau)f(\tau)}{\sum_{\tau\in\mathcal{T}^Q_{M_2,2}}D(\tau)}
\end{align}
where $f(\tau)$ is the feature of trajectory $\tau$. Furthermore, we also have the expectation of $g_Z$ as
\begin{align}
    \mathbb{E}[g_Z]&=\mathbb{E}_{Q}[\frac{P(\tau|\theta,T)}{Q(\tau|\theta)}\nabla_{\theta}R_{\theta}(\tau)]\\
    &=\int P(\tau|\theta,T)f(\tau)\\
    &=\mathbb{E}_{P}[f(\tau)]
\end{align}
Therefore, because of
\begin{align*}
    g_{\theta}&=\frac{1}{M_1}\sum_{\tau\in \mathcal{D}_{M_1}} f(\tau)-g_Z\\
    &=\frac{1}{M_1}\sum_{\tau\in \mathcal{D}_{M_1}} f(\tau)-\frac{\sum_{\tau\in\mathcal{T}^{Q}_{M_2,1}}D(\tau)f(\tau)}{\sum_{\tau\in\mathcal{T}^Q_{M_2,2}}D(\tau)}
\end{align*}
we can see the expectation of $g_{\theta}$ is
\begin{align}
    \mathbb{E}[g_{\theta}]=\mathbb{E_{\mathcal{D}}}[f(\tau)]-\mathbb{E}_{P}[f(\tau)]=\nabla_{\theta}\log Z_{\theta}
\end{align}
This completes the proof.
\end{proof}

\section{Proof of Theorem 3}\label{app:Th:main}
To prove Theorem 3, we first introduce a Lemma as follows:
\begin{Lm}\label{Lm:L-smooth}
If the total variation $\max_{\theta}Var_{P}(f(\tau))\leq \sigma_2^2$, then $L(\theta)$ is $\sigma_2^2$-smooth w.r.t. $\theta$.
\end{Lm}

\begin{proof}
Since $L(\theta)=-\frac{1}{N}\sum_{\tau\in\mathcal{D}}\log P(\tau|\theta,T)$, $L$-smoothness requires that
\begin{align*}
    ||\nabla L(\theta_1)-\nabla l(\theta_2)||_2\leq L||\theta_1-\theta_2||_2
\end{align*}
where $L$ is a constant. Because of the mean value theorem, there exists a point $\tilde{\theta}\in(\theta_1, \theta_2)$ such that
\begin{align*}
    \nabla L(\theta_1)-\nabla L(\theta_2)=\nabla(\nabla L(\tilde{\theta}))(\theta_1-\theta_2).
\end{align*}
Taking the $L_2$ norm for both sides, we have
\begin{align}
    ||\nabla L(\theta_1)-\nabla L(\theta_2)||_2=&
    ||\nabla(\nabla L(\tilde{\theta}))(\theta_1-\theta_2)||_2\nonumber\\
    \leq&||\nabla(\nabla L(\tilde{\theta}))||_2~||\theta_1-\theta_2||_2\label{eq:lem11}
\end{align}
Then, the problem is to bound the matrix 2-norm $||\nabla(\nabla L(\tilde{\theta}))||_2$. Since we know the explicit form of  $L(\theta)$, we know
\begin{align}
    \nabla L(\theta)&=\nabla\log Z_{\theta}-\frac{1}{N}\sum_{\tau\in\mathcal{D}}f(\tau),    \nonumber\\
    \nabla(\nabla L(\theta))&= \sum_{\tau\in\mathcal{T}}[f(\tau)-\nabla\log Z_{\theta}][f(\tau)-\nabla\log Z_{\theta}]^T P(\tau|\theta,T)\label{eq:lem12},
\end{align}
where $\nabla(\nabla L(\theta))$ is the co-variance matrix. Denote Cov$_{\theta}[f(\tau)]=\nabla(\nabla L(\theta))$, which is both symmetric and positive semi-definite. We have
\begin{align*}
    ||\nabla(\nabla L(\tilde{\theta}))||_2=||\text{Cov}_{\theta}[f(\tau)]||_2=\lambda_{max},
\end{align*}
where $\lambda_{max}$ is the maximum eigenvalue of the matrix Cov$_{\theta}[f(\tau)]$. Then, because of the positive semi-definiteness of the co-variance matrix, all the eigenvalues are non-negative, and we can bound $\lambda_{max}$ as
\begin{align*}
    \lambda_{max}\leq\sum_{i}\lambda_i=Tr(\text{Cov}_{\theta}[f(\tau)]),
\end{align*}
where $Tr($Cov$_{\theta}[\phi(X)])$ is the trace of matrix Cov$_{\theta}[f(\tau)]$.  Using the definition in Equation \ref{eq:lem12}, $Tr($Cov$_{\theta}[f(\tau)])$ can be further derived as: 
\begin{align*}
    Tr(\text{Cov}_{\theta}[f(\tau)])=\mathbb{E}_{P}[||f(\tau)||_2^2]-||\mathbb{E}_{P}[f(\tau)]||_2^2,
\end{align*}
which is equal to the total variation $Var_{P}(f(\tau))$. Therefore, we have
\begin{align*}
    ||\nabla(\nabla L(\tilde{\theta}))||_2\leq Var_{P}(f(\tau))\leq \sigma_2^2.
\end{align*}
Combining this with Equation \ref{eq:lem11}, we know 
\begin{align*}
    ||\nabla L(\theta_1)-\nabla L(\theta_2)||_2\leq \sigma_2^2~||\theta_1-\theta_2||_2.
\end{align*}
This completes the proof.

\end{proof}

We give the full proof of Theorem 3 as follows:

\begin{proof}(Theorem 3)
Since we use $M_1$ samples from the training set $\{\tau_i\}_{i=1}^{M_1}$ and $2M_2$ samples $\tau'_1, \dots, \tau'_{M_2}$ from $Q(\tau|\theta)$ using XOR-Sampling at each iteration, we have 
$$g_k = \frac{1}{M_1}\sum_{j=1}^{M_1}f(\tau_j)-\frac{\sum_{j=1}^{M_2}D(\tau'_j)f(\tau'_j)}{\sum_{j=M_2+1}^{2M_2}D(\tau'_j)}$$
Assume these samples $\tau'$ form distribution $Q'$, and denote $g_k^i=\frac{1}{M_1}\sum_{j=1}^{M_1}f(\tau_j)-\frac{D(\tau'_i)f(\tau'_i)}{\mathbb{E}_{Q'}[D(\tau)]}$, we have the expectation of $g_k$ as
\begin{align*}
\mathbb{E}_{\mathcal{D},Q'}[g_k]&=\frac{\mathbb{E}_{Q'}[\mathbb{E}_{\mathcal{D}}[f(\tau)]\mathbb{E}_{Q'}[D(\tau)]-D(\tau')f(\tau')]}{\mathbb{E}_{Q'}[D(\tau)]}\\
&=\mathbb{E}_{\mathcal{D},Q'}[g_k^i].
\end{align*}

In each iteration $k$ we can adjust the parameters in XOR-Sampling to give the constant factor approximation of both the denominator and the nominator, then for each $g_k^i$ we can obtain from Theorem 2 that
\begin{align}
    \frac{1}{\delta^2} [\nabla L(\theta_k)]^+ &\leq \mathbb{E}_{\mathcal{D},Q'}[g_k^{i+}]\leq \delta^2 [\nabla L(\theta_k)]^+,\label{eq:pos2}\\
    \delta^2 [\nabla L(\theta_k)]^- &\leq \mathbb{E}_{\mathcal{D},Q'}[g_k^{i-}]\leq  \frac{1}{\delta^2}[\nabla L(\theta_k)]^-.\label{eq:neg2}
\end{align}
where $\nabla L(\theta_k)$ is the true gradient at $k$-th iteration. Denote $g_k^+=\max\{g_k,\textbf{0}\}$ and $g_k^-=\min\{g_k,\textbf{0}\}$. 
Clearly, $g_k^{i+} \geq 0$ and $g_k^{i-} \leq 0$. Moreover,
for a given dimension, either $g_k^{i+}=0$ for that dimension or $g_k^{i-}=0$. 
Evaluating $g_k$ dimension by dimension, we can see that
$g_k^+=\frac{1}{M_2}\sum_{i=1}^{M_2} g_k^{i+}$ and $g_k^-=\frac{1}{M_2}\sum_{i=1}^{M_2} g_k^{i-}$.
Combined with Equation~\ref{eq:pos2} and \ref{eq:neg2}, we know 
\begin{align*}
    \frac{1}{\delta^2} [\nabla L(\theta_k)]^+ \leq \mathbb{E}[g_k^+]
    \leq \delta^2 [\nabla L(\theta_k)]^+,\\
    \delta^2 [\nabla L(\theta_k)]^- \leq \mathbb{E}[g_k^-]
    \leq \frac{1}{\delta^2} [\nabla L(\theta_k)]^-.
\end{align*}

In terms of variance, assume the variance of each $f(\tau'_j)$ can be bounded by $Var_{P}(f(\tau'_j))\leq \sigma_2^2$ resulted from both importance sampling and XOR-Sampling, and because $\mathbb{E}_{\mathcal{D},P}[g_k]=\mathbb{E}_{\mathcal{D},P}[g_k^i]$, the variance of $g_k$, denoted as $Var_{\mathcal{D},P}(g_k)$, can then be bounded as
\begin{align*}
    Var_{\mathcal{D},P}(g_k)&= Var_D(\frac{1}{M_1}\sum_{j=1}^{M_1}f(\tau_j)) +\\
    & Var_{P}(\frac{1}{M_2}\sum_{i=1}^{M_2}f(\tau'_i))\\
    & =\frac{1}{M_1}Var_D(f(\tau_j)) + \frac{1}{M_2}Var_{P}(f(\tau'_i))\\
    & \leq \frac{\sigma_1^2}{M_1} + \frac{\sigma_2^2}{M_2}
\end{align*}

Therefore, since $L(\theta)$ is convex and $\sigma_2^2-$smooth from Lemma \ref{Lm:L-smooth}, we can then apply Theorem 4 to get the result in Theorem 3.
\begin{align*}
     \mathbb{E}[L(\overline{\theta_K})]-OPT &\leq \frac{\delta^2||\theta_0-\theta^*||_2^2}{2\eta K}+\frac{\eta\max_{\theta_k}\{Var_{\mathcal{D},P}(g_k)\}}{\delta^2}\\
     &\leq\frac{\delta^2||\theta_0-\theta^*||_2^2}{2\eta K}+\frac{\eta\sigma_1^2}{\delta^2 M_1}+\frac{\eta\sigma_2^2}{\delta^2M_2}.
\end{align*}
This completes the proof.
\end{proof}

\section{Proof of Theorem 5}
\begin{proof} (Theorem 5)
Since we use flow constraints to ensure valid trajectories, the number of binary variables in XOR-Sampling in then $O(|\mathcal{S}||\mathcal{A}|)$. 
From Theorem 2 we know that in each iteration of X-MEN, we need to access $O(|\mathcal{S}||\mathcal{A}|\ln\frac{|\mathcal{S}||\mathcal{A}|}{\gamma})$ queries of NP oracles in order to generate one sample. However, as specified also in \cite{ermon2013embed}, only the first sample needs those many queries. Once we have the first sample, the number of XOR constraints to add can be known in generating future samples for this SGD iteration. 
Therefore, we fix the number of XOR constraints added starting the generation of the second sample.
As a result, we only need one  NP oracle query in generating each of the following $2(M_2-1)$ samples. 
Therefore, total queries in each iteration will be $O(|\mathcal{S}||\mathcal{A}|\ln\frac{|\mathcal{S}||\mathcal{A}|}{\gamma}+M_2)$. To complete all $K$ SGD iterations,  X-MEN needs $O(K|\mathcal{S}||\mathcal{A}|\ln\frac{|\mathcal{S}||\mathcal{A}|}{\gamma}+KM_2)$ NP oracle queries in total.
\end{proof}

\section{ADAPT SAMPLE DISTRIBUTION VIA VARIANCE REDUCTION}
We have shown that one can solve the constrained inverse reinforcement learning problem by estimating the expectation term in maximum likelihood learning by importance sampling, where a proposal distribution is used to get valid samples that satisfy additional constraints via XOR-Sampling. In this section we show that parameter $\hat{\theta}$ of the proposal distribution can be adaptively updated based on variance reduction. Denote the importance weight $w(\tau)=\frac{P(\tau|\theta,T)}{Q(\tau|\hat{\theta})}$. Since we use importance sampling, the expectation $\mu_q=\mathbb{E}[g_Z]$ is an unbiased estimator of the true value $\mu=\nabla_{\theta}\log Z_{\theta}$, and the variance is $\sigma_Q^2/n$ where 
\begin{align}
    \sigma_Q^2 &= \sum_{\tau} \frac{\nabla_{\theta}R(\theta)^2P(\tau|\theta,T)^2}{Q(\tau|\hat{\theta})}-\mu^2 \nonumber\\
    &= \mathbb{E}_{Q}[\nabla_{\theta}R(\theta)^2w(\tau)^2]-\mu^2
\end{align}
When the reward function $R_{\theta}(\tau)$ is represented by a linear combination of hand-crafted features, i.e., $R_{\theta}(\tau)=\theta^Tf(\tau)$, we let the sample distribution $Q$ share weights with the nominal distribution $P$, i.e., $\hat{\theta}=\theta$, which means each iteration we update parameters $\theta$ in $P$, we update parameters as $\hat{\theta}=\theta$ at the same time. However, in a more general case when reward function $R_{\theta}(\tau)$ is represented not in a linear form, such as a neural network presentation, we can not simply let $\hat{\theta}=\theta$. In this case, we need to find the optimal $\hat{\theta}$ by optimization on some loss functions. One possible way is to minimize the variance via stochastic gradient descent, namely,
\begin{align}
    \hat{\theta}=\text{argmin}_{\hat{\theta}}\sigma_Q^2=\text{argmin}_{\hat{\theta}}\mathbb{E}_{Q}[\nabla_{\theta}R(\theta)^2w(\tau)^2]
\end{align}
In this paper we only let $R_{\theta}(\tau)=\theta^Tf(\tau)$ where $f(\tau)$ is either hand-crafted features or neural networks. However, even in this case we can still use the variance reduction method to update $\hat{\theta}$. We leave this version of updating $\hat{\theta}$ for future work where we can see it might not still hold the theoretical guarantee of convergence rate, yet might work better in practice due to the strong representation power of deep neural network.

\section{Human Obstacle Avoidance}

\begin{figure}[t]
\centering
\includegraphics[width=0.5\linewidth]{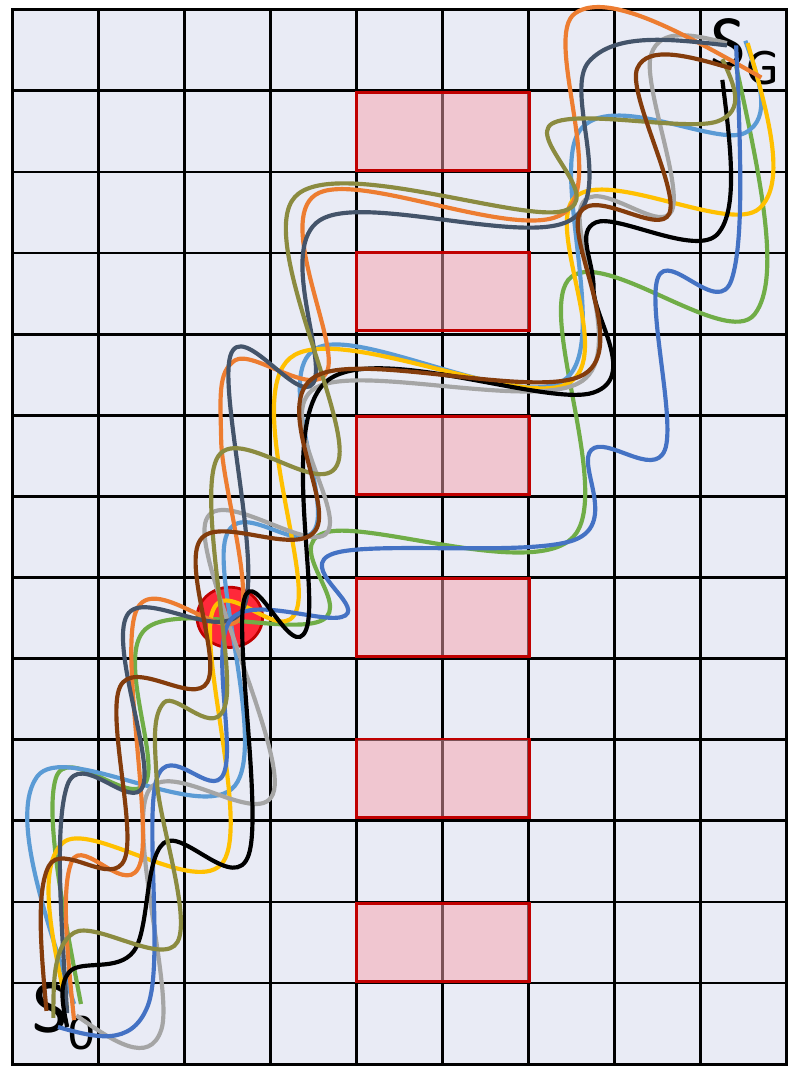}
\caption{Overlaid trajectories generated by X-MEN to learn the human preferences. The goal is to move from $S_0$ to $S_G$ and the action space contains only going up and right. The shaded regions represent obstacles in the human’s environment, and the red circle represent a "must pass" point. Additional constraints are that human cannot take the same action consecutively for 3 times. We can see the generated trajectories from X-MEN satisfy all the constraints and follow the shortest possible paths, similar to what human demonstrators' actions.}\label{fig:obstacle}
\vspace*{-0.3cm}
\end{figure}

 The human agents are attempting to reach a fixed goal state $S_G$ from a given initial state $S_0$, as shown in Figure \ref{fig:obstacle} in the appendix. The agent has only two actions for moving up or moving right. The shaded regions represent obstacles in the human’s environment that cannot be passed through, and the red circle represent a "must pass" choke point that every person has to walk through. Additional hard constraints are that human cannot take the same action consecutively for more than 3 times. 
 
 We only show the generated trajectories of X-MEN in Figure \ref{fig:obstacle}. We can see X-MEN is able to successfully avoid obstacles and pass the "must go" choke point, and the 10 generated trajectories shown in the figure are indeed the shortest paths from the start state to the goal (matching human demonstrations). What is worth noting is that X-MEN learns to go up first before passing through the gap between two obstacles, because otherwise the trajectory has to violate the constraint of taking the same action consecutively for no more than 3 times.